\documentclass[11pt,twoside,a4paper]{article}
\usepackage{times}
\usepackage{epsfig}
\usepackage{amsmath}
\usepackage{amsthm}
\usepackage{amsfonts}
\usepackage{amssymb}
\usepackage{tabularx}
\usepackage{graphicx}
\usepackage{epstopdf}
\usepackage{algorithm}
\usepackage{algorithmic}
\newtheorem{theorem}{Theorem}

\newcommand{\bm}{\mathbf m}
\DeclareMathOperator{\Tr}{Tr}
\DeclareMathOperator*{\argmin}{arg\,min}

\begin{document}
\title{Bayesian Fisher's Discriminant for Functional Data}
\author{
Yao-Hsiang Yang \footnote{Institute of Information Science, Academia Sinica, Taipei, Taiwan, yhyang@statistics.twbbs.org }, Lu-Hung Chen\footnote{Institute of Statistics, National Chung-Hsing University, Taichung, Taiwan, luhung@nchu.edu.tw. Corresponding author.}, Chieh-Chih Wang\footnote{Institute of Networking and Multimedia, National Taiwan University, Taipei, Taiwan, bobwang@csie.ntu.edu.tw}, and Chu-Song Chen \footnote{Institute of Information Science, Academia Sinica, Taipei, Taiwan, song@iis.sincia.edu.tw} 
}
\maketitle

\begin{abstract}

We propose a Bayesian framework of Gaussian process in order to extend Fisher's discriminant to classify functional data such as spectra and images. The probability structure for our extended Fisher's discriminant is explicitly formulated, and we utilize the smoothness assumptions of functional data as prior probabilities. Existing methods which directly employ the smoothness assumption of functional data can be shown as special cases within this framework given corresponding priors while their estimates of the unknowns are one-step approximations to the proposed MAP estimates. Empirical results on various simulation studies and different real applications show that the proposed method significantly outperforms the other Fisher's discriminant methods for functional data.
\end{abstract}

\textit{Keywords: Fisher's linear discriminant analysis, functional data, Gaussian process, Bayesian smoothing, dimension reduction}

\section{Introduction}\label{intro}

Dimension reduction techniques have become standard workhorses in the fields of pattern recognition and computer vision, which stand either directly for the purpose of recognition or indirectly for pre-processing. Two of the most popular and fundamental methods among this ever-growing family, namely Principal Component Analysis (PCA) and Linear Discriminant Analysis (LDA), both stem directly from the classical works in the field of multivariate statistical analysis. However, one must notice that most of these standard dimension reduction techniques only work well when the training sample size $n$ is large enough, especially when $n$ is much larger than the number of features $p$.

When $p$ is larger than $n$ (recently, \cite{HDLSS} coined the term \emph{High Dimension Low Sample Size}, abbreviated to \emph{HDLSS} to describe this situation), dimension reduction seems to be an intuitive solution to high-dimensional data. For example, Eigenface \cite{eigenface} and Fisherface \cite{fisherface} (which were famous in computer vision and pattern recognition societies) use PCA and LDA to perform image recognition tasks. However, the empirical results are usually far from satisfactory. In fact, these intuitive ideas could be misleading: it has been shown theoretically in the statistics society that most of the standard dimension reduction techniques  such as PCA and LDA may not be applicable for HDLSS data. For example, \cite{pcahdlss} showed that PCA may be merely a random projection in the HDLSS context, while \cite{ldahd} showed that LDA may not work better than merely random guesses when the dimension is larger than the sample size.

Although HDLSS data may arise from lots of situations, we focus on a specific problem: the data are discretization of stochastic (smooth) functions. That is, the multivariate data we obtained are first randomly drawn from some functional space (with certain smoothness assumption) and then sampled (or discretized) by some digital sampler. Such kind of data are often referred to as \emph{functional data}. Functional data arises in lots of scientific and engineering applications. For example, images are often regarded as discretization of "smooth" functions (which is the idea of image denoising), and the number of image pixels is usually much larger than the number of images. Most examples in the field of digital signal processing can be regarded as functional data. Some other examples for functional data includes mass spectrum in chemistry or energy spectrum in physics. 

The most important property of functional data is that the observations should have strong spatial or temporal correlations if they are discretized from smooth functions. Such property has been used to generalize standard dimension reduction techniques to functional data, including (but not limited to) Functional PCA \cite{fpca}, Functional Canonical Correlation Analysis (Functional CCA) \cite{fcca}, Penalized Discriminant Analysis (PDA) \cite{pda}, etc. \cite{LSSS} introduced PDA to the computer vision society and showed that PDA easily outperforms the methods based on classical multivariate methods such as Eigenface and Fisherface. Although these extensions do show great improvement over standard dimension reduction techniques for functional data, they incorporate the smoothness property of functional data in an intuitive way. In this article, we aim to provide a Bayesian framework to utilize the smoothness property of functional data in a more \emph{theoretical} fashion.

In this article we use Fisher's LDA to illustrate our Bayesian framework for two reasons. First, \cite{fpls_hall} showed that LDA (asymptotically) achieves optimal discrimination for functional data under mild conditions, even when the Gaussian assumption is violated. Hence, we don't need to compare all the possible discrimination techniques. Second, the performances of discrimination tasks are easy to evaluate: they can be evaluated by their classification accuracies. 

The contribution of this article is of two-fold: first, we propose a Bayesian smoothing approach for functional LDA and show that the proposed approach achieves best prediction accuracy among existing modifications of LDA for functional data. This indicates that our proposed Bayesian framework utilizes the smoothness property of functional data better than the other existing methods. Second, the Bayesian approach provide an easier way to select the required tuning parameters. Most of the existing approaches for functional data select their tuning parameters by cross-validation and thus require lots of computational efforts. On the other hand, the tuning parameters in Bayesian framework can be \emph{estimated} simultaneously with the other unknown parameters. Hence our proposed Bayesian approach is less computationally intensive than the existing approaches as well.

We also argue that kernel methods such as kernel principal component analysis (kernel PCA, \cite{kpca}), kernel discriminant analysis (KDA, \cite{kda}), kernel support vector machine, etc., may not be an appropriate solution for analyzing functional data. The idea of kernel trick is to project the finite dimensional multivariate data (often in $\mathbb{R}^p$) to an infinite dimensional functional space specified by a certain kernel, and then perform dimension reduction techniques on the functional space. However, when data are observed from functions, they already lie on some functional space, and projecting them to another functional space may not be helpful. The same argument has been pointed out by \cite{fsvm} as well: they showed that linear SVM often performs better than kernel SVMs for functional data. Later in our empirical studies, we will also show that kernel tricks usually does not work well for functional data no matter which kernel function is chosen.

The rest of the article is arranged as follows. Recent advances on the dimension reduction of functional data, especially for linear discriminant analysis, are reviewed in section \ref{review}. We also provide a very brief introduction to the idea of Bayesian smoothing in section \ref{review}, too. In section \ref{gplda} we introduce the proposed Bayesian framework for LDA to model functional data and show that Penalized Discriminant Analysis \cite{pda} is actually an approximation within this framework given corresponding priors. The maximum-a-posterior (MAP) estimation of the unknown parameters within the proposed framework and the selection of required tuning parameters are presented in section \ref{gplda}, too. The performances of the proposed approach are demonstrated by simulation studies in section \ref{simulation} and real-world applications in section \ref{realapp}. The conclusion remarks are drawn in section \ref{conclusion}.

\section{Literature Review}\label{review}
Recall that the Fisher's discriminant directions are the solutions of the following generalized eigenproblem:
\begin{equation}
\hat{\beta} \equiv \arg\max_{\beta}\frac{\beta^{T}\hat{\Sigma}_{b}\beta}{\beta^{T}\hat{\Sigma}_{w}\beta}.
\label{Fisherdir}
\end{equation}
where $\Sigma_{w}$ is the common covariance matrix among all classes, $\Sigma_{b}$ is the between covariance matrix which characterizes the variance of the class means:
\[
\Sigma_{b}\equiv\sum_{i=1}^{c}{\left(\mu_{i}-\bar{\mu}\right)\left(\mu_{i}-\bar{\mu}\right)^{T}},
\]
$\mu_i$'s are the means of the respective classes, and $\bar{\mu}$ is the overall mean. ($\hat{\Sigma}_{b}$ and $\hat{\Sigma}_{w}$ are corresponding estimates of $\Sigma_{b}$ and $\Sigma_{w}$.) Although Fisher's discriminant does not rely on any probability assumptions of the data, usually the parameters are still estimated by maximum likelihood estimation (MLE) which maximizes the following log-likehood function:
\begin{equation}
\ln{ L\left(\mu_i,\Sigma_{w}\right)}\propto-\frac{1}{2}\sum_{i=1}^c\sum_{j=1}^{n_i}{\left(y_{i,j}-\mu_{i}\right)^{T}\Sigma_{w}^{-1}\left(y_{i,j}-\mu_{i}\right)}-\frac{n}{2}\ln{\left|\Sigma_{w}\right|},
\label{MLE}
\end{equation}
where the log-likelihood is up to a constant and $\left\{y_{i,j}\right\}$ is the set of training samples of size $n=n_{1}+n_{2}+\cdots+n_c$. It is easy to derive that the maximizers of (\ref{MLE}) are
\begin{align*}
\hat{\mu}_{i}&=\frac{1}{n_{i}}\sum_{j=1}^{n_i}{y_{i,j}},\hspace{100pt} i=1,\cdots,c, \\
\hat{\Sigma}_{w}&=\frac{1}{n}\sum_{i=1}^c\sum_{j=1}^{n_{i}}{\left(y_{i,j}-\hat{\mu}_{i}\right)^{T}\left(y_{i,j}-\hat{\mu}_{i}\right)}.
\end{align*}
That is, although Fisher's LDA does not require that the data is generated from normal distributions, we still use the normality assumption $y_{i,j}\sim N(\mu_i,\Sigma_{w})$ as a working assumption.

\subsection{Fisher's LDA for Functional Data}

There is a rich literature of extending Fisher's LDA for functional data. Most methods can be categorized into three categories: the \emph{pre-smoothing} approaches that apply smoothing/denoising filters to the data before performing LDA, the \emph{regularized} approaches that generalize the idea of ridge LDA, and the \emph{basis-representation} approaches that project the data to certain functional bases before performing Fisher's LDA. 

The idea for \emph{pre-smoothing} approaches is quite intuitive: since the observations are from some smooth functions, one can apply some smooth filter (e.g. Gaussian smoothing filter) to each observation, and perform LDA to the smoothed data. However, applying smoothing filters to each observations require lots of computation efforts, since the tuning parameters for smoothing filters need to be selected for each smoothing. Moreover, \cite{CarrollDH13} showed that the tuning parameters should be selected to under-smooth each observations for the pre-smoothing approaches to achieve reasonble results (compared to the other classifiers for functional data). However, choosing such tuning parameters that satisfy the requirements in \cite{CarrollDH13} is non-trivial in practice. Besides, the dimension for the smoothed data is still very high, and hence the \emph{HDLSS} problem as mentioned in Section \ref{intro} still remains

The \emph{regularized} approach was introduced by \cite{pda}, and they named it as Penalized Discriminant Analysis (PDA). Their idea is also naive: when the sample size $n$ is larger than the dimension $p$, the estimated covariance matrix $\hat{\Sigma}_W$ is usually stable and nonsingular. However, when $p>n$, $\hat{\Sigma}_W$ becomes singular (since the rank of $\hat{\Sigma}_W$ is $\min\{n,p\}$) and noisy. Hence it is natural to regularize $\hat{\Sigma}_W$ with some full rank matrix, such as the identity matrix $I$ used in \cite{rlda}, to stabilize the outcomes of Fisher's LDA. The \emph{regularized} Fisher's discriminant directions become the solutions of:
\begin{equation}
\hat{\beta} \equiv \arg\max_{\beta}\frac{\beta^{T}\hat{\Sigma}_{b}\beta}{\beta^{T}\left(\hat{\Sigma}_{w}+\alpha\Omega\right)\beta},
\label{PDAdir}
\end{equation}
where $\Omega$ is a regularization matrix and $\alpha\in[0,\infty)$ is the tuning parameter. For functional data, the regularization matrix $\Omega$ is usually derived from finite-difference matrices (discretized version of some differential operator). For example, Hastie et al. used $\Omega=D_2^TD_2$ for 1-dimensional functions such as audio signals, where $D_2$ is the second-order finite-difference matrix
\[
D_2=
\left[\begin{array}{cccccc}
1 & -2 & 1 & 0 & \cdots & 0 \\
0 & 1 & -2 & 1 & \ddots  & \vdots \\
\vdots & \ddots & \ddots & \ddots & \ddots & 0\\
0 & \cdots & 0 & 1 & -2 & 1
\end{array}
\right]_{(n-2)\times n},
\]
and use $\Omega=\Delta$ where $\Delta$ is the descritzed Laplacian operator for 2-dimensional functions such as images \cite{pda}. 

The regularized approach is simple and can be easily extended to incorporate with nonlinear manifolds. For example, Cai et al. introduced PDA to the computer vision society and incorporate it with nonlinear manifolds such as ISOMAP, locally linear embedding, etc. \cite{LSSS}. However, regularized approaches usually do not perform as well as basis-representation approaches due to their careless modelings. In section \ref{gplda} we will show that regularized approaches such as PDA are inappropriate special cases of our Bayesian Gaussian process framework; in sections \ref{simulation} and \ref{realapp} we will show that regularized approaches can outperform filtered approaches with careful modeling. Moreover, the performance of PDA is sensitive to the tuning parameter $\alpha$. Usually the tuning parameter $\alpha$ is selected by cross-validation. However, cross-validation requires lots of computation efforts, especially when the dataset is large.

The \emph{basis-representation} approach assumes that the smooth functions which generate our observed (discretized) data can be represented by $q$ basis functions, with usually $q\ll n$. If this assumption holds, we can project our data to the space spanned by the $q$ basis functions, and discriminate the data by their coefficients (scores) of the basis functions. The basis functions can be either predefined such as B-spline, Fourier, wavelets, etc., or data-driven such as principal components analysis (PCA), partial least squares (PLS), etc. For example, \cite{James01} suggested a B-spline approach for functional LDA; \cite{berlinet2008functional} introduced a wavelets approach for classifying functional data. However, Basis representations may introduce additional problems as well. For example, \cite{James01} showed in their article that their approach may be confounding and leads to local optimum. Furthermore, the predefined basis functions may not represent the observed data well, and choosing ``good'' basis functions to represent the observed data still remains a challenge. 

On the other hand, data-driven basis functions are derived by the observed data and usually provide better representations to the data we observed. Hence, data-driven basis representations are more preferable in recent literatures. For example, \cite{fpca_logistic} introduced a functional PCA approach for classifying 1-dimensional curves; \cite{PredaCS07} presented a functional PLS approach for functional LDA; \cite{fpls_hall} illustrated an example that PLS is more preferable than the other basis representations for discrimination tasks: if the differences between group means cannot be represented by the selected basis functions, discriminating by the basis coefficients will be no better than merely random guess. They showed that PLS will never find such basis functions that cannot represent the differences of group means, while PCA may find such basis functions when the total covariance is dominated by the within covariance. Later in section \ref{simulation} we will further discuss this situation by simulation examples.

The above ideas for Fisher's LDA can be also extended to the other dimension reduction techniques for functional data. For example, \cite{PredaCS07} and \cite{fpls_hall} utilize the \emph{pre-smoothing} idea for their functional PLS; \cite{ppca} and \cite{ppls} utilize the \emph{regularized} approach for functional PCA and PLS; \cite{Aguilera96} utilized the \emph{basis-representation} approach for functional PCA. Since we focus on Fisher's LDA in this article, we will not provide a detailed survey for the other dimension reduction techniques for functional data.

\subsection{A Short Introduction to Bayesian Smoothing}\label{bsmooth}

Bayesian smoothing is a Bayesian perspective for smoothing and denoising. Let's illustrate it with the following toy example: assume that we have observed the pairs $(x_i,y_i)$ that come from
\begin{equation}
	y_i=m(x_i)+\epsilon_i,\hspace{56pt} i=1,\cdots,n,
	\label{smoothing}
\end{equation}
where $m$ is an unknown \emph{smooth} function and $\epsilon_i$ are random noises with mean $0$. The key idea of smoothing/denoising techniques is to find a balance between the distance of $y_i$ and $m(x_i)$ with the \emph{smoothness} of $m(\cdot)$. In general, this idea can be described as the following optimization problem
\begin{equation}
	\hat{m}=\argmin\sum_{i=1}^{n}\|y_i-m(x_i)\|+\alpha \mathcal{J}(m),
	\label{smoothingopt}
\end{equation}
where $\mathcal{J}(m)$ denotes some measurement of the smoothness for function $m$ and $\alpha$ is a tuning parameter that controls the ``degree of smoothness'' of $m$. Typical choices for the distance between $y_i$ with $m(x_i)$ are $\ell_1$ and $\ell_2$ norms, and $\mathcal{J}(m)$ is usually described by the $\ell_1$ and $\ell_2$ norms of $m'(\cdot)$ or $m''(\cdot)$. As $m(\cdot)$ is an unknown function, the norms of $m'$ or $m''$ by the Riemman sum of the finite difference of $m$. For example, the $\ell_1$ norm of $m'(\cdot)$ is estimated by
\begin{equation*}
	\|m'(x)\|_1=\int |m'(x)|dx\approx\sum_{i=1}^{n-1}|m(x_{i+1}-m(x_{i})|=\|D\bm\|_1,
\end{equation*}
where $\bm=(m(x_1),\cdots,m(x_n))^{T}$ and 
\[
D=
\left[
	\begin{array}{ccccc}-1 & 1 & 0 & \cdots & 0 \\
						0 & -1 & 1 & \ddots & \vdots \\
						\vdots & \ddots & \ddots & \ddots & 0 \\
						0 & \cdots & 0 & -1 & 1 \\
	\end{array}
\right]_{(n-1)\times n}
\]
is the first-order finite difference operator. When both norms are taken to be $\ell_2$ norms, (\ref{smoothingopt}) becomes 
\begin{equation}
\begin{split}
	\argmin&\sum_{i=1}^{n}\left(y_i-m(x_i)\right)^2+\alpha\bm^{T}D^{T}D\bm \\
	=&\sum_{i=1}^{n}\left(y_i-m(x_i)\right)^2+\alpha\bm^{T}\Omega\bm,
\end{split}
\label{smoothingspline}
\end{equation}
which becomes the famous smoother \emph{linear smoothing spline} \cite{dssp}; when both norms are taken to be $\ell_1$ norms, (\ref{smoothingopt}) becomes another famous denoising technique \emph{L1-TV} \cite{l1tv}.

The idea of Bayesian smoothing is to describe the smoothing/denoising problems as probabilistic models and replace the optimization problem (\ref{smoothing}) with a \emph{maximum a posteriori} (MAP) estimation. This is usually accomplished by assuming that $\epsilon_i$ follow a certain probability distribution (e.g. i.i.d Normal distribution) and assigning a prior distribution to characterize ``the smoothness of $m$''. For example, if we assume that $\epsilon_i\stackrel{iid}{\sim}N(0,\sigma_{\epsilon}^{2})$ and use the prior probability $\bm\sim N(0,\sigma_{m}^{2}\Omega)$ with $\Omega=D^{T}D$ for $\bm$, then the log-posterior of (\ref{smoothing}) becomes
\begin{equation}
	-\frac{1}{2\sigma_{\epsilon}^2}\sum_{i=1}^{n}\left(y_i-m(x_i)\right)^2-\frac{1}{2\sigma_m^2}\bm^T\Omega\bm
\label{bayesmoothingspline}
\end{equation}
up to some constants irrelevant to the parameters of major interest $\bm$. Maximizing the log-posterior distribution (\ref{bayesmoothingspline}) is equivalent to minimizing
\[
	\sum_{i=1}^{n}\left(y_i-m(x_i)\right)^2+\alpha\bm^T\Omega\bm
\]
with $\alpha=\sigma_{\epsilon}^2/\sigma_m^2$, which yields to the same linear smoothing spline smoother as in (\ref{smoothingspline}). Similarly, if we assume that $\epsilon_i\stackrel{iid}{\sim}Laplae(0,\sigma_{\epsilon})$ and the prior probability $\bm\sim Laplace(0,\sigma_{m}\Omega)$ for $\bm$, we will derive a MAP estimation that is equivalent to L1-TV denoising. For more introduction to Bayesian smoothing/denoising techniques, we refer to the following references: \cite{Gribonval11,BayesTV}. For more applications and technical details about Bayesian smoothing, we suggest the comprehensive textbook by \cite{bayesmooth}.

\subsection{Bayesian Fisher's LDA}

\begin{table}[htdp]
	\centering
		\begin{tabular}{|lp{0.5\linewidth}l|}
		  \hline
		  Notation & Meaning & Analytic Expression\\
		  \hline
			$c$ & Number of classes & \\
			$n_i$ & Sample size of $i$-th class & \\
			$n$ & Total sample size & \\
			$p$ & Feature space dimension (resolution for image) & \\
			$\mu_{i}$ & Mean function of $i$-th class & \\
			$z_{i,j}$ & $j$-th sample function of $i$-th class & \\
			$x_{i,j}$ & A $p$-dimensional vector which is sampling from $z_{i,j}$ & \\
			$y_{i,j,k}$ & Noisy observation corresponding to $k$-th component of $x_{i,j}$ & \\
			$\Sigma_{w}$ & Within-class covariance matrix & \\
			$\Sigma_{b}$ & Between-class covariance matrix & $\sum_{i=1}^{c}{\left(\mu_{i}-\bar{\mu}\right)\left(\mu_{i}-\bar{\mu}\right)^{T}}$\\
&&\\
			$D$ & $1$-st order difference matrix & $\left[
	\begin{array}{ccccc}-1 & 1 & 0 & \cdots & 0 \\
						0 & -1 & 1 & \ddots & \vdots \\
						\vdots & \ddots & \ddots & \ddots & 0 \\
						0 & \cdots & 0 & -1 & 1 \\
	\end{array}
\right]_{(n-1)\times n}$\\
&&\\
			$D_2$ & $2$-nd order difference matrix & $\left[\begin{array}{cccccc}
1 & -2 & 1 & 0 & \cdots & 0 \\
0 & 1 & -2 & 1 & \ddots  & \vdots \\
\vdots & \ddots & \ddots & \ddots & \ddots & 0\\
0 & \cdots & 0 & 1 & -2 & 1
\end{array}
\right]_{(n-2)\times n}$\\
&&\\
			$\Omega$ & Smoothing matrix & e.g. $D^{T}D$, $D_2^{T}D_2$\\
			\hline
		\end{tabular}
	\caption{Reference for Our Notations}
	\label{notation}
\end{table}

The phrase \emph{Bayesian Fisher's LDA} has also been used in some literatures \cite{BayesLSSVM,BayesKDA}. However, these articles focus on Bayesian interpretations for Kernel Discriminant Analysis (KDA), i.e. applying kernel tricks to Fisher's LDA. They assume that the observations follow Gaussian process after kernel transformations, and apply some priors (e.g. Gaussian prior) to the KDA coefficients. However, as we will show in our empirical studies (Sections \ref{simulation} and \ref{realapp}), KDA does not perform as well as the other approaches for functional data. To our knowledge, our work is the \emph{first} article that discusses Bayesian formulations of Fisher's LDA for functional data.

\section{The Gaussian Process Model of Fisher's Discriminant}\label{gplda}

In this section, we introduce our Bayesian modeling strategy of Fisher's LDA for functional data. We assume that the the observed data are digitized from some underlying smooth random functions with some i.i.d Gaussian white noises, and assume that the underlying functions follows Gaussian processes with unknown mean functions and a common covariance function as parameters. Then we assign certain prior probabilities inspired by the idea of Bayesian smoothing to the unknown mean functions, and a Wishart prior (also inspired by the idea of Bayesian smoothing) to the unknown common covariance function. We describe our Gaussian Process data generation model in Section \ref{gpmodel}, derive the priors from the idea of Bayesian smoothing in Section \ref{priors} as well as the corresponding maximum a posteriori (MAP) estimation in Section \ref{map}. The selection of tuning parameters is also addressed as an estimation problem and are estimated simultaneously with the mean and covariance functions in Sections \ref{priors} and \ref{map}. We suggest a back-fitting algorithm for the MAP estimation and prove the convergence of this algorithm in Section \ref{backfitting}. Finally, we show that penalized discriminant analysis (PDA) can be interpreted as an (improper) special case of our Bayesian framework in Section \ref{PDA}. Since this section contains most derivation, we make a reference table for our notations in Table \ref{notation} in the beginning of this section.

\subsection{Gaussian Process Model}\label{gpmodel}

For finite dimensional data, Fisher's LDA models each class on multivariate normal distribution with different means and the same within-class covariance. The Gaussian Process model is an natural extension for multivariate normal distribution when the data are observed from continuous functions. We assume all our data are generated from GP model with different mean functions and the same within-class covariance function. More specifically, the $j$-th function $z_{i,j}(\cdot)$ from $i$-th class follows Gaussian processes $GP\left(\mu_i(\cdot),\Sigma_w(\cdot,\cdot)\right)$ distribution with \emph{mean functions} $\mu_i(\cdot)$ and the common \emph{covariance function}s $\Sigma_w(\cdot,\cdot)$. Of course, we only have discrete observations after sampling. We denote the discrete observation of function $z_{i,j}(\cdot)$ by $x_{i,j}$ which is a $p$-dimensional vector sampled from $z_{i,j}(\cdot)$ with $x_{i,j}=(z_{i,j}(t_1),\ldots,z_{i,j}(t_p))$, where $\{t_k|k=1,\ldots,p\}$ are the sample positions.

Assume further the observation error is additive white Gaussian noise and the corresponding discrete observations are $y_{i,j}$. The $k$-th component, which is corresponding to the point $t_k$, of vector $x_{i,j}$ is denoted by $x_{i,j,k}$ and that of vector $y_{i,j}$ by $y_{i,j,k}$. By definition, for any Gaussian process $GP\left(\mu_i(\cdot),\Sigma_w(\cdot,\cdot)\right)$ and any index $\{t_k|k=1,\ldots,p\}$ of multivariate Gaussian distribution $N\left(\mu_i,\Sigma_w\right)$. Here we slightly abuse the notations for letting $\mu_i=(\mu(t_1),\ldots,\mu(t_p))^{T}$ and $\Sigma_w=\left(\Sigma_w(t_i,t_j)\right)_{p\times p}$.

Now the noisy discrete observations $\{y_{i,j,k}\}$ follow:
\begin{align*}
&y_{i,j} = x_{i,j}+\epsilon_{i,j}, \hspace{53pt} i=1,\ldots,c,j=1,\ldots,n_c,\\
&x_{i,j} \sim N(\mu_i,\Sigma_{w}), \hspace{48pt} i=1,\ldots,c,j=1,\ldots,n_c,\\
&\epsilon_{i,j} \sim N(0,\sigma^{2}I_{p}), \hspace{48pt} i=1,\ldots,c,j=1,\ldots,n_c,\\
\end{align*}
where $I_{p}$ is the $p\times p$ identity matrix.

\subsection{Priors for the Mean and Covariance Functions}\label{priors}

Following the most conventional practice of Bayesian statistics, we assign conjugate priors to the mean functions and the within-class covariance (e.g. \cite{wishart} employ a similar Bayesian hierarchical model for estimation.). Given feature dimension $p$, the priors for the $p$-dimensional mean vectors $\mu_i$ and the $p\times p$ dimensional within-covariance matrix $\Sigma_w$ follow normal and inverse-Wishart distribution, respectively, which are the corresponding conjugate priors:

\begin{equation}
\begin{split}
\pi_{m}(\mu_i)&= \frac{1}{2\pi^{\frac{k}{2}}}|\alpha_1\Omega|^{\frac{1}{2}} e^{-\frac{1}{2}\alpha_1\mu_i^T\Omega\mu_i}, \\
\pi_{wc}(\Sigma_w) &= \frac{|\alpha_2\Omega|^{\frac{\nu}{2}}}{2^{\frac{\nu p}{2}} \Gamma_{p}(\frac{\nu}{2})}|\Sigma_w|^{-\frac{\nu+p+1}{2}}e^{-\frac{1}{2}\Tr(\alpha_2\Omega\Sigma_{w}^{-1})},
\end{split}
\label{gpprior}
\end{equation}
where $\Omega=D^{T}D$ and $D$ is the discretized version of the first-order differential operator $d$. Inspired by Bayesian smoothing technique introduced in Section \ref{bsmooth}, we specify the prior parameter $\Omega=D^{T}D$ for both normal and inverse-Wishart priors. 

Now the remaining unspecified parameters are $\nu$, $\alpha_1$, $\alpha_2$ and $\sigma^2$. We adapt the hierarchical Bayesian approach and follow the conventional estimation procedure for inverse-Wishart distribution which re-parametrize the degree of freedom as $\delta = \nu - (p-1)$ since $\nu > p-1$ (\cite{wishart}). For $\alpha_1$, $\alpha_2$ and $\frac{1}{\sigma^2}$, we further assign the conjugate (hyper-)priors to them. The conjugates prior is the distribution whose corresponding posterior distribution belongs to the same family which thus simplifies the form of posterior. In our case, all these parameters shall follow gamma distribution $\Gamma\left(a,b\right)=\frac{b^a}{\Gamma(a)}x^{a-1}e^{-bx}$. We represent our hierarchical Bayesian model in Figure \ref{network}.


\begin{figure}
	\centering
		\includegraphics[scale=0.6]{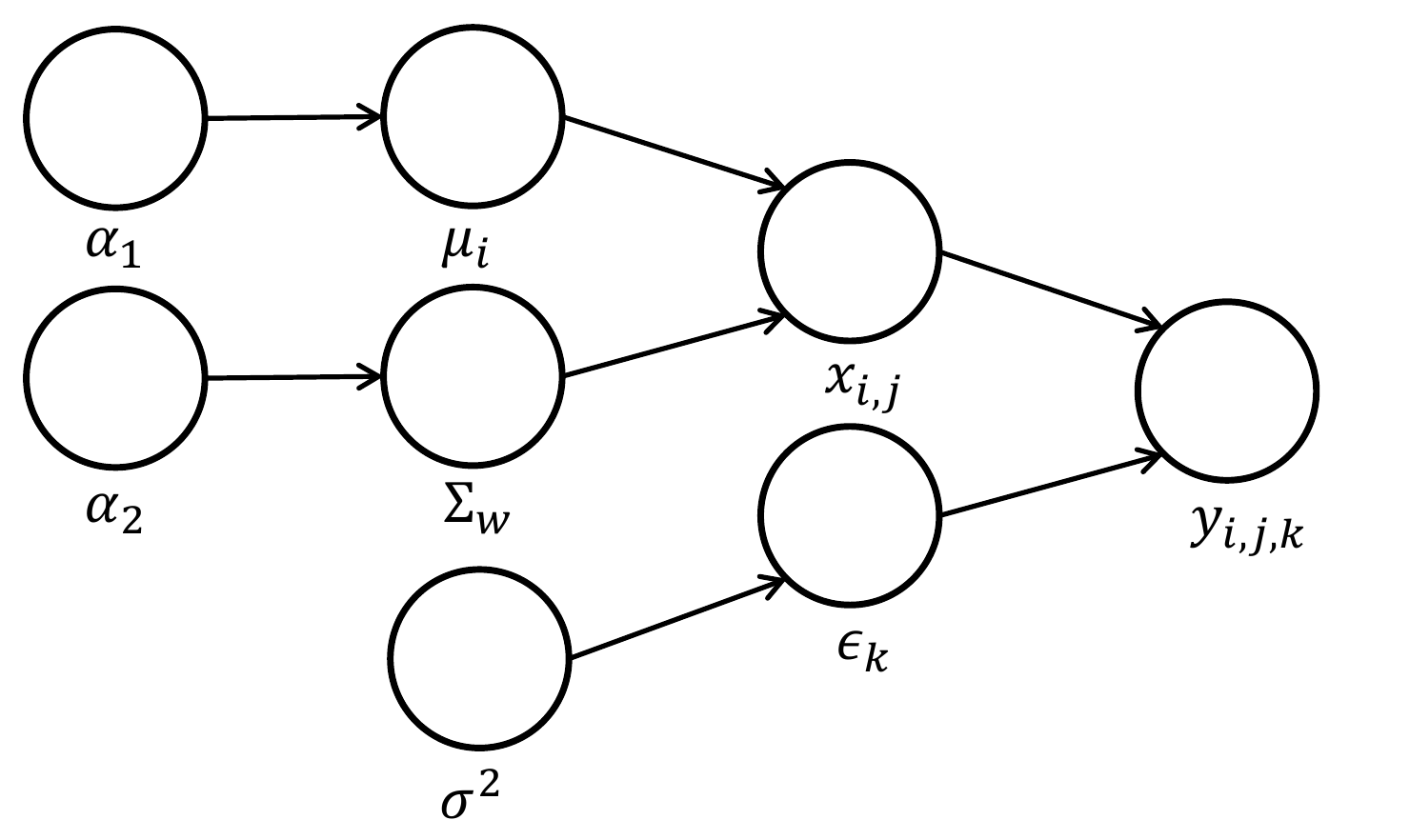}
	\caption{Graph representation of our Bayesian model}
		\label{network}
\end{figure}

\subsection{MAP estimation}\label{map}

We now have the log-posterior:
\begin{align*}
&l\left(\{x_{i,j}\},\{\mu_i\},\Sigma_w,\alpha_1,\alpha_2,\frac{1}{\sigma^2}\bigg|\{y_{i,j}\}\right)\\
=&\log p\left(\{y_{i,j}\}\bigg|\{x_{i,j}\},\frac{1}{\sigma^2}\right)+\log p\left(\{x_{i,j}\}|\{\mu_i\},\Sigma_w\right)+\log \pi_{m}\left(\{\mu_i\}|\alpha_1\right)+\log \pi_{wc}\left(\Sigma_w|\alpha_2\right)\\
&+\log \pi\left(\alpha_1\right)+\log \pi\left(\alpha_2\right)+\log \pi\left(\frac{1}{\sigma^2}\right)+C,
\end{align*}
which follows the graphical model in Figure \ref{network} and $C$ is a constant for normalization. After incorporating the Gaussian process assumption and the above priors, we have:
\begin{equation}
\begin{split}
&2l\left(\{x_{i,j}\},\{\mu_i\},\Sigma_w,\alpha_1,\alpha_2,\frac{1}{\sigma^2}\bigg|\{y_{i,j}\}\right)\\
=&-\sum_{i=1}^c\sum_{j=1}^{n_i}{\frac{\left(y_{i,j}-x_{i,j}\right)^{T}\left(y_{i,j}-x_{i,j}\right)}{\sigma^{2}}}+np\ln{\frac{1}{\sigma^{2}}}-\sum_{i=1}^c\sum_{j=1}^{n_i}{\left(x_{i,j}-\mu_{i}\right)^{T}\Sigma_{w}^{-1}\left(x_{i,j}-\mu_{i}\right)}-n\ln{\left|\Sigma_{w}\right|}\\
&-\alpha_1\sum_{i=1}^c\mu_i^T\Omega\mu_i+c\ln\alpha_1-\alpha_2\Tr\left(\Omega\Sigma_{w}^{-1}\right)+p\ln\alpha_2-\left(\nu+p+1\right)\ln{\left|\Sigma_{w}\right|}\\
&-2b_1\alpha_1+2(a_1-1)\ln\alpha_1-2b_2\alpha_2+2(a_2-1)\ln\alpha_2-2b_3\frac{1}{\sigma^{2}}+2(a_3-1)\ln{\frac{1}{\sigma^{2}}}+C.
\end{split}
\label{gppost}
\end{equation}

Given the log-posterior stated above, the maximum a posteriori (MAP) estimation of (\ref{gppost}) must satisfy the following first-order condition:
\begin{align}
&\left\{\begin{array}{l} 
\displaystyle \frac{\partial 2l}{\partial \alpha_1}=-\sum_{i=1}^c\mu_i^T\Omega\mu_i+c\frac{1}{\alpha_1}-2b_1+2(a_1-1)\frac{1}{\alpha_1}=0,\\
\displaystyle \frac{\partial 2l}{\partial \alpha_2}=-\Tr\left(\Omega\Sigma_{w}^{-1}\right)+p\frac{1}{\alpha_2}-2b_2+2(a_2-1)\frac{1}{\alpha_2}=0,\\
\displaystyle \frac{\partial 2l}{\partial (1/\sigma^2)}=-\sum_{i=1}^c\sum_{j=1}^{n_i}{\left(y_{i,j}-x_{i,j}\right)^{T}\left(y_{i,j}-x_{i,j}\right)}+np\sigma^{2}-2b_3+2(a_3-1)\sigma^{2}=0,\\
\displaystyle \frac{\partial 2l}{\partial x_{i,j}}=-\frac{1}{\sigma^2}\left(x_{i,j}-y_{i,j}\right)-\Sigma_{w}^{-1}\left(x_{i,j}-\mu_{i}\right)=0, \hspace{70pt} \forall i=1,\ldots,c,\forall j=1,\ldots,n_c,\\
\displaystyle \frac{\partial 2l}{\partial \mu_i}=\sum_{j=1}^{n_{i}}{\Sigma_{w}^{-1}\left(\mu_{i}-x_{i,j}\right)}+\alpha_1\Omega\mu_{i}=0, \hspace{106pt} \forall i=1,\ldots,c,\\
\displaystyle \frac{\partial 2l}{\partial \Sigma_{w}}=n\left[\Sigma_{w}^{-1}\left(\frac{1}{n}\sum_{i=1}^{c}\sum_{j=1}^{n_{i}}{\left(x_{i,j}-\mu_{i}\right)\left(x_{i,j}-\mu_{i}\right)^{T}}\right)\Sigma_{w}^{-1}-\frac{1}{\rho}\Sigma_{w}^{-1}+\frac{1}{n}\alpha_2\Sigma_{w}^{-1}\Omega\Sigma_{w}^{-1}\right]=0.
\end{array}\right.
\label{M-system}
\end{align}
The last equation can be derived by the matrix calculus formulae: $\frac{\partial \Tr \left(AX^{-1}\right)}{\partial X}=X^{-1}A^{T}X^{-1}$, $\frac{\partial \ln \left|X\right|}{\partial X}=\left(X^{-1}\right)^T$, since 
\[
\sum_{i=1}^c\sum_{j=1}^{n_i}{\left(x_{i,j}-\mu_{i}\right)^{T}\Sigma_{w}^{-1}\left(x_{i,j}-\mu_{i}\right)}=\Tr\left(\sum_{i=1}^c\sum_{j=1}^{n_i}{\left(x_{i,j}-\mu_{i}\right)}\left(x_{i,j}-\mu_{i}\right)^{T}\Sigma_{w}^{-1}\right),
\]
and both $\Omega$ and $\Sigma_{w}$ are symmetric.

Now solving the above linear system (\ref{M-system}) yields:
\begin{align}
&\left\{\begin{array}{l} 
\displaystyle\alpha_1=\frac{2a_1+c-2}{2b_1+\sum_{i=1}^{c}{\mu_{i}^{T}\Omega\mu_{i}}},\\
\displaystyle\alpha_2=\frac{2a_2+p-2}{2b_2+\Tr\left(\Omega\Sigma_{w}^{-1}\right)},\\
\displaystyle\sigma^2=\frac{2b_3+\sum_{i=1}^c\sum_{j=1}^{n_i}{\left(y_{i,j}-x_{i,j}\right)^{T}\left(y_{i,j}-x_{i,j}\right)}}{a_3+np-2},\\
\displaystyle x_{i,j}=\left(\Sigma_{w}+\sigma^2I\right)^{-1}\left(\Sigma_{w}y_{i,j}+\sigma^2\mu_{i}\right), \hspace{80pt} \forall i=1,\ldots,c,\forall j=1,\ldots,n_c,\\
\displaystyle\mu_{i}=\left(I+\frac{\alpha_1}{n_{i}}\Sigma_{w}\Omega\right)^{-1}\bar{x}_{i}, \hspace{136pt} \forall i=1,\ldots,c,\\
\displaystyle\Sigma_{w}=\rho S_{xx}+\frac{\rho}{n}\alpha_2\Omega,
\end{array}\right.
\label{M-solution}
\end{align}
where $\bar{x}_{i}=\frac{1}{n_i}\sum_{j=1}^{n_{i}}{x_{i,j}}$, $S_{xx}=\sum_{i=1}^c\sum_{j=1}^{n_{i}}{\left(x_{i,j}-\mu_{i}\right)\left(x_{i,j}-\mu_{i}\right)^{T}}$, $\rho=\frac{n}{n+\nu+p+1}$.

\subsection{Backfitting Algorithm and Convergence Results}\label{backfitting}

From (\ref{M-system}) we notice that the estimation of $\mu_i$ requires the information about $\Sigma_{w}$, and vice versa. Thus iteration is required to solve the system (\ref{M-system}). We propose the following backfitting algorithm to estimate $\mu_i$ and $\Sigma_{w}$:
\nopagebreak
\begin{algorithm}[H]
\caption{Backfitting Algortihm}
\label{algbf}
\begin{algorithmic}[1]
\STATE Initial $\hat{\mu}_i$ by $\bar{y}_i$ and $\hat{\Sigma}_{w}$ by .
\REPEAT 
\STATE $\hat{\alpha}_1=\frac{2a_1+c-2}{2b_1+l\sum_{i=1}^{c}{\hat{\mu}_{i}^{T}\Omega\hat{\mu}_{i}}}.$
\STATE $\hat{\alpha}_2=\frac{2a_2+p-2}{2b_2+\Tr\left(\Omega\hat{\Sigma}_{w}^{-1}\right)}$
\STATE $\hat{\sigma}^2=\frac{2b_3+\sum_{i=1}^c\sum_{j=1}^{n_i}{\left(y_{i,j}-\hat{x}_{i,j}\right)^{T}\left(y_{i,j}-\hat{x}_{i,j}\right)}}{a_3+np-2}$
\STATE $\hat{x}_{i,j}=\left(\hat{\Sigma}_{w}+\hat{\sigma}^2I\right)^{-1}\left(\hat{\Sigma}_{w}y_{i,j}+\hat{\sigma}^2\hat{\mu}_{i}\right), \hspace{80pt} \forall i=1,\ldots,c$
\STATE $\hat{\mu}_{i}=\left(I+\hat{\Sigma}_{w}\frac{\hat{\alpha}_1}{n_{i}}\Omega\right)^{-1}\bar{\hat{x}}_{i}$
\STATE $\hat{\Sigma}_{w}=\rho\left(\sum_{i=1}^c\sum_{j=1}^{n_i}{\frac{1}{n}\left(\hat{x}_{i,j}-\hat{\mu}_{i}\right)\left(\hat{x}_{i,j}-\hat{\mu}_{i}\right)^{T}}\right)+\frac{\rho}{n}\hat{\alpha}_2\Omega$
\UNTIL converges
\end{algorithmic}
\end{algorithm}

The important question for an iteration algorithm is whether it converges in finite steps or not. We now show that the backfitting algorithm converges almost surely when $\{n_i|i=1,\ldots,c\}$ are sufficiently large under the following conditions:

\noindent {\bf Condition ${\cal C}$}:
\begin{enumerate}
\item[(${\cal C}$1)] $a_1>\frac{1}{2}M-\frac{1}{2}c+1$.
\item[(${\cal C}$2)] $a_2>\frac{1}{2}\Tr\Omega-\frac{1}{2}p+1$.
\end{enumerate}
where $M=2M'_{0}Qcp\max_{i,k}{\left\{\|\Omega\hat{\mu}_{i,(k)}\|_{2},\|\Omega\hat{\mu}_{i}\|_{2}\right\}}$, $M'_{0}=\max_{k}\{|\hat{\alpha}_{1,(k)}|,|\hat{\alpha}_{1}|\}^2$, $Q=\frac{M_2}{2(1-M_1M'_2)}$ which will be specified below and $\|\cdot\|_{2}$ is the operator norm of matrix: $\|A\|_{2}=\sup_{x\neq0}\frac{\|Ax\|_{2}}{\|x\|_{2}}$.

\begin{theorem}
Let $\hat{\mu}_{i,(k)}$, $\hat{\Sigma}_{w,(k)}$, $\hat{x}_{i,j,(k)}$, $\hat{\alpha}_{1,(k)}$, $\hat{\alpha}_{2,(k)}$ and $\hat{\sigma}_{(k)}^2$ be the estimate of $\mu_{i}$, $\Sigma_{w}$, $x_{i,j}$, $\alpha_1$, $\alpha_2$ and $\sigma^2$ at $k$th iteration in algorithm \ref{algbf}, $\hat{\mu}_{i}$, $\hat{\Sigma}_{w}$, $\hat{x}_{i,j}$, $\hat{\alpha}_{1}$, $\hat{\alpha}_{2}$ and $\hat{\sigma}^2$ be the MAP estimate of $\mu_{i}$, $\Sigma_{w}$, $x_{i,j}$, $\alpha_1$, $\alpha_2$ and $\sigma^2$ satisfying Eq. (\ref{M-system}). Suppose $a_1,a_2$ further satisfy Condition ${\cal C}$ stated above. We have:
\[
\hat{\mu}_{i,(k)}\rightarrow\hat{\mu}_{i},\hat{\Sigma}_{w,(k)}\rightarrow\hat{\Sigma}_{w}, \hat{x}_{i,j,(k)}\rightarrow\hat{x}_{i,j},\hat{\alpha}_{1,(k)}\rightarrow\hat{\alpha}_{1}, \hat{\alpha}_{2,(k)}\rightarrow\hat{\alpha}_{2}, \hat{\sigma}_{(k)}^2\rightarrow\hat{\sigma}^2, \tag*{a.s.}
\]
for $\{n_i|i=1,\ldots,c\}$ are sufficiently large.
\end{theorem}
\begin{proof}
We first derive the bound of $\hat{\alpha}_{1,(k)}$ and $\hat{\alpha}_{2,(k)}$ from the Condition ${\cal C}$.
\begin{align*}
&a_1>\frac{1}{2}M-\frac{1}{2}c+1\\
\Rightarrow & \frac{1}{M'_{0}Qc}\sum_i{\|\hat{\mu}_{i,(k-1)}-\hat{\mu}_{i}\|_{2}}>\left|\sum_i{\frac{\left(\hat{\mu}_{i,(k-1)}-\hat{\mu}_{i}\right)\Omega\left(\hat{\mu}_{i,(k-1)}+\hat{\mu}_{i}\right)}{2a_1+c-2}}\right|\\
&\hspace{100pt}=\left|\frac{\sum_i{\hat{\mu}_{i,(k-1)}\Omega\hat{\mu}_{i,(k-1)}}-\sum_i{\hat{\mu}_{i}\Omega\hat{\mu}_{i}}}{2a_1+c-2}\right|\geq\left|\frac{1}{\hat{\alpha}_{1,(k)}}-\frac{1}{\hat{\alpha}_{1}}\right|\\
\Rightarrow &\left|\hat{\alpha}_{1,(k)}-\hat{\alpha}_{1}\right|<\frac{1}{Qc}\sum_i{\|\hat{\mu}_{i,(k-1)}-\hat{\mu}_{i}\|_{2}},\\
&a_2>\frac{1}{2}\Tr\Omega-\frac{1}{2}p+1\\
\Rightarrow &\left|\Tr(\hat{\Sigma}_{w,(k-1)}^{-1}-\hat{\Sigma}_{w}^{-1})\right|>\frac{\Tr\Omega}{2a_2+p-2}\left|\Tr(\hat{\Sigma}_{w,(k-1)}^{-1}-\hat{\Sigma}_{w}^{-1})\right|\\
&\hspace{100pt}\geq\left|\frac{\Tr(\Omega\hat{\Sigma}_{w,(k-1)}^{-1})-\Tr(\Omega\hat{\Sigma}_{w}^{-1})}{2a_2+p-2}\right|=\left|\frac{1}{\hat{\alpha}_{2,(k)}}-\frac{1}{\hat{\alpha}_{2}}\right|\\
\Rightarrow &\left|\hat{\alpha}_{2,(k)}-\hat{\alpha}_{2}\right|<M_0\left\|\hat{\Sigma}_{w,(k-1)}-\hat{\Sigma}_{w}\right\|_{2},
\end{align*}
for some constant $M_0$, since $\left\|\hat{\Sigma}_{w,(k-1)}^{-1}\right\|_{2}\left\|\hat{\Sigma}_{w,(k-1)}-\hat{\Sigma}_{w}\right\|_{2}\left\|\hat{\Sigma}_{w}^{-1}\right\|_{2}\geq\left\|\hat{\Sigma}_{w,(k-1)}^{-1}(\hat{\Sigma}_{w,(k-1)}-\hat{\Sigma}_{w})\hat{\Sigma}_{w}^{-1}\right\|_{2}\geq\frac{1}{p}\left|\Tr(\hat{\Sigma}_{w,(k-1)}^{-1}-\hat{\Sigma}_{w}^{-1})\right|$.

Then for $\hat{\mu}_{i,(k)}$, we have:
\begin{align*}
\hat{\mu}_{i,(k)}&=\left(I+\hat{\Sigma}_{w,(k-1)}\frac{\hat{\alpha}_{1,(k)}}{n_{i}}\Omega\right)^{-1}\bar{\hat{x}}_{i,(k)}\\
&=\left(I+\hat{\Sigma}_{w}\frac{\hat{\alpha}_{1}}{n_{i}}\Omega+\left(\hat{\alpha}_{1}\hat{\Sigma}_{w}-\hat{\alpha}_{1,(k)}\hat{\Sigma}_{w,(k-1)}\right)\frac{1}{n_i}\Omega\right)^{-1}\bar{\hat{x}}_{i,(k)}\\
&=\hat{\mu}_{i}+\left(I+\hat{\Sigma}_{w}\frac{\hat{\alpha}_{1}}{n_{i}}\Omega\right)^{-1}\left(\bar{\hat{x}}_{i,(k)}-\bar{\hat{x}}_{i}\right)+A^{-1}\left(\hat{\alpha}_{1}\hat{\Sigma}_{w}-\hat{\alpha}_{1,(k)}\hat{\Sigma}_{w,(k-1)}\right)B\frac{1}{n_i}\Omega A^{-1}\bar{\hat{x}}_{i,(k)},
\end{align*}
where $A=I+\hat{\Sigma}_{w}\frac{\hat{\alpha}_1}{n_{i}}\Omega$ and $B=\left(I+\frac{1}{n_i}\Omega A^{-1}\left(\hat{\alpha}_{1}\hat{\Sigma}_{w}-\hat{\alpha}_{1,(k-1)}\hat{\Sigma}_{w,(k-1)}\right)\right)^{-1}$. The last equation is derived from the matrix inversion formula. Since all matrices are bounded in the last equation, $\|\hat{\mu}_{i,(k)}-\hat{\mu}_{i}\|_{2} \leq M_1\|\bar{\hat{x}}_{i,(k)}-\bar{\hat{x}}_{i}\|_{2}+M_2|\hat{\alpha}_{1,(k)}-\hat{\alpha}_{1}|+M_3\left\|\hat{\Sigma}_{w,(k-1)}-\hat{\Sigma}_{w}\right\|_{2}$ for some constants $M_1$, $M_2$ and $M_3$. Here $M_1<1$ since the smallest eigenvalue of $\Omega^{t}\frac{\hat{\alpha}_{1}}{n_{i}}\hat{\Sigma}_{w}^{t}+\hat{\Sigma}_{w}\frac{\hat{\alpha}_{1}}{n_{i}}\Omega$ is larger than $0$.

Since
\begin{align*}
&\hat{x}_{i,j,(k)}=\left(\hat{\Sigma}_{w,(k-1)}+\hat{\sigma}_{(k-1)}^2I\right)^{-1}\left(\hat{\Sigma}_{w,(k-1)}y_{i,j}+\hat{\sigma}_{(k-1)}^2\hat{\mu}_{i,(k-1)}\right)\\
\Rightarrow &\bar{\hat{x}}_{i,(k)}=\left(\hat{\Sigma}_{w,(k-1)}+\hat{\sigma}_{(k-1)}^2I\right)^{-1}\left(\hat{\Sigma}_{w,(k-1)}\bar{y}_{i}+\hat{\sigma}_{(k-1)}^2\hat{\mu}_{i,(k-1)}\right).
\end{align*}
By Strong Law of Large Number and MLE consistency, we have $\|\sum_{j=1}^{n_i}{\frac{1}{n_{i}}\left(y_{i,j}-\hat{\mu}_{i}\right)}\|_{2}<\epsilon_1$ for all $i$ when $\{n_i|i=1,\ldots,c\}$ are sufficiently large. Now $\bar{\hat{x}}_{i,(k)}$ can be bounded by
\begin{align*}
\|\bar{\hat{x}}_{i,(k)}-\bar{\hat{x}}_{i}\|_{2} &\leq M'_1\epsilon_1+\|\left(\hat{\Sigma}_{w,(k-1)}+\hat{\sigma}_{(k-1)}^2I\right)^{-1}\hat{\sigma}_{(k-1)}^2\left(\hat{\mu}_{i,(k-1)}-\hat{\mu}_{i}\right)\|_{2}\\
&\leq M'_1\epsilon\|\left(\hat{\Sigma}_{w,(k-1)}-\hat{\Sigma}_{w}\right)\|_{2} + M'_2\|\left(\hat{\mu}_{i,(k-1)}-\hat{\mu}_{i}\right)\|_{2}
\end{align*}
for some constants $M'_1 = \left\|\left(\hat{\Sigma}_{w,(k-1)}+\hat{\sigma}_{(k-1)}^2I\right)^{-1}\right\|_{2}$ and $M'_2=\left\|\left(\hat{\Sigma}_{w,(k-1)}+\hat{\sigma}_{(k-1)}^2I\right)^{-1}\hat{\sigma}_{(k-1)}^2I\right\|_2<1$.

Similarly, Since
\begin{align*}
\frac{1}{\rho}\left(\hat{\Sigma}_{w,(k)}-\hat{\Sigma}_{w}\right)=\frac{1}{n}\sum_{i=1}^c\sum_{j=1}^{n_i}&{\left(x_{i,j}-\hat{\mu}_{i}\right)\left(\hat{\mu}_{i}-\hat{\mu}_{i,(k)}\right)^{T}+\left(\hat{\mu}_{i}-\hat{\mu}_{i,(k)}\right)\left(x_{i,j}-\hat{\mu}_{i}\right)^{T}}\\
&+\frac{1}{n}\sum_{i=1}^c{n_i\left(\hat{\mu}_{i}-\hat{\mu}_{i,(k)}\right)\left(\hat{\mu}_{i}-\hat{\mu}_{i,(k)}\right)^{T}}+\frac{1}{n}\left(\hat{\alpha}_{2,(k)}-\hat{\alpha}_{2}\right)\Omega.
\end{align*}
where $\rho=\frac{n}{n+\nu+p+1}$. For any $\epsilon_2$, by Strong Law of Large Number and MLE consistency again, $\|\sum_{j=1}^{n_i}{\frac{1}{n_{i}}\left(x_{i,j}-\hat{\mu}_{i}\right)}\|_{2}<\epsilon_2$ for all $i$ in probability $1$ when $\{n_i|i=1,\ldots,c\}$ are sufficiently large. This implies
\begin{align*}
&\left\|\hat{\Sigma}_{w,(k)}-\hat{\Sigma}_{w}\right\|_{2}\\
\leq&2\epsilon_2\frac{\rho}{n}\left(\sum_{i}{\|\hat{\mu}_{i,(k)}-\hat{\mu}_{i}\|_{2}}\right)+\frac{1}{n}\sum_{i}{n_{i}\|\hat{\mu}_{i,(k)}-\hat{\mu}_{i}\|_{2}}^2+\frac{\|\Omega\|_{2}}{n}|\hat{\alpha}_{2,(k)}-\hat{\alpha}_{2}|\\
\leq &2\epsilon_2\frac{\rho}{n}\left(\sum_{i}{\|\hat{\mu}_{i,(k)}-\hat{\mu}_{i}\|_{2}}\right)+\frac{1}{n}\sum_{i}{n_{i}\|\hat{\mu}_{i,(k)}-\hat{\mu}_{i}\|_{2}}^2+\frac{\|\Omega\|_{2}M_0}{n}\left\|\hat{\Sigma}_{w,(k-1)}-\hat{\Sigma}_{w}\right\|_{2}.
\end{align*}

Now we have
\begin{align*}
&\sum_{i}{\|\hat{\mu}_{i,(k)}-\hat{\mu}_{i}\|_{2}} \\
\leq& \sum_{i}{\left(M_1\|\bar{\hat{x}}_{i,(k)}-\bar{\hat{x}}_{i}\|_{2}+M_2|\hat{\alpha}_{1,(k)}-\hat{\alpha}_{1}|+M_3\left\|\hat{\Sigma}_{w,(k-1)}-\hat{\Sigma}_{w}\right\|_{2}\right)}\\
\leq& \left(M_1M'_2+\frac{M_2}{Q}+2(cM_1M'_1\epsilon_1+M_3)\epsilon_2\frac{\rho}{n}\right)\sum_{i}{\|\left(\hat{\mu}_{i,(k-1)}-\hat{\mu}_{i}\right)\|_{2}}\\
&\hspace{10pt}+(cM_1M'_1\epsilon_1+M_3)\left(\frac{1}{n}\sum_{i}{n_{i}\|\hat{\mu}_{i,(k)}-\hat{\mu}_{i}\|_{2}}^2+\frac{1}{n}\|\Omega\|_{2}M_0\left\|\hat{\Sigma}_{w,(k-1)}-\hat{\Sigma}_{w}\right\|_{2}\right).
\end{align*} 
Since $M_1<1$, $M'_2<1$, and $Q=\frac{M_2}{2(1-M_1M'_2)}$, we have $M_1M'_2+\frac{M_2}{Q}<1$. Since $\epsilon_1, \epsilon_2$ can be taken arbitrarily small, $\hat{\mu}_{i,(k)}$ and $\hat{\Sigma}_{w,(k)}$ converges almost surely when $\{n_i|i=1,\ldots,c\}$ are sufficiently large. Convergence of the rest parameters can be easily deduced from these.
\end{proof}

\subsection{Penalized Discriminant Analysis (PDA) as a Special Case}
\label{PDA}

If we impose on the eigenfunctions of $\Sigma_W$ the following prior,
\begin{equation}
  \pi(\gamma_i)\propto e^{-\alpha\frac{1}{\lambda_{i}}\left\|d{\gamma_{i}}\right\|^{2}},
  \label{hbtprior}
\end{equation}
where $\gamma_{i}$ and $\lambda_{i}$ are eigenfuntions and eigenvalues of $\Sigma_{w}$ respectively and $d$ is a differential operator (for example, the Laplacian operator $\Delta$ in Cai et al. \cite{LSSS}), together with the Gaussian process assumption stated in Section 2, the corresponding maximum a posterior estimators become the solution of
\begin{equation}
\min_{\mu_{i},\Sigma_{w}}\sum_{i=1}^c{\sum_{j=1}^{n_i}{\frac{1}{n}\left(y_{i,j}-\mu_{i}\right)^{T}\Sigma_{w}^{-1}\left(y_{i,j}-\mu_{i}\right)}}+\ln{\left|\Sigma_{w}\right|}+\mbox{Tr}\left(\alpha\Omega\Sigma_{w}^{-1}\right),
\label{HBT}
\end{equation}
where $\Omega=D^TD$ and $D$ is the discretized version of the operator $d$.

And hence, by the first-order condition of (\ref{HBT}) we have:
\begin{eqnarray}
  \left\{\begin{array}{c}
    \displaystyle\sum_{j=1}^{n_i}{\Sigma_{w}^{-1}\left(y_{i,j}-\mu_{i}\right)}=0,\\
    \displaystyle -\Sigma_{w}^{-1}\left(\sum_{i=1}^c\sum_{j=1}^{n_i}{\frac{1}{n}\left(y_{i,j}-\mu_{i}\right)\left(y_{i,j}-\mu_{i}\right)^{T}}\right)\Sigma_{w}^{-1}+\Sigma_{w}^{-1}-\Sigma_{w}^{-1}\alpha\Omega\Sigma_{w}^{-1}=0,
  \end{array}\right.\\
  \Rightarrow\left\{\begin{array}{c}
    \displaystyle\hat{\mu}_{i}=\sum_{j=1}^{n_i}{y_{i,j}}=\bar{y}_{i},\\
    \displaystyle\hat{\Sigma}_{w}=\left(\sum_{i=1}^c\sum_{j=1}^{n_i}{\frac{1}{n}\left(y_{i,j}-\bar{y}_{i}\right)\left(y_{i,j}-\bar{y}_{i}\right)^{T}}\right)+\alpha\Omega=S_{yy}+\alpha\Omega,
  \end{array}\right.
\end{eqnarray}
where $\bar{y}_{i}$ are the sample class means and $S_{yy}$ is the sample covariance. The above equation implies that the discriminant direction is:
\begin{equation}
  \hat{\beta}\equiv\arg\max_{\beta}{\frac{\beta^{T}\hat{\Sigma}_{b}\beta}{\beta^{T}\left(S_{yy}+\alpha\Omega\right)\beta}},
\end{equation}
which is exactly the criterion of PDA in (\ref{PDAdir}).

Note that PDA do not impose any prior structure on mean functions $\mu_i$, the estimation of means $\mu_i$ and hence $\Sigma_{b}$ shall be less accurate than our proposed method. Since the estimation of $\Sigma_W$ also depends on the estimation of means, this estimate shall also be less accurate as well.

\section{Simulation Studies}\label{simulation}

In this section we investigate the performance of the proposed GP-LDA by conducting two simulations. We compare our GP-LDA approach with kernel discriminant analysis (KDA) with different kernel functions (radius basis function kernel, polynomial kernels with order 3, 4, and 5), penalized discriminant analysis (PDA), basis representation approaches including B-spline based approach \cite{James01}, functional PCA and functional PLS (that is, performing LDA after functional PCA/PLS). Among different approaches of functional PCA, we choose the penalized PCA (PPCA hereafter) by \cite{ppca} since it is one of the best functional PCA approach and is the only method that can be applied to multidimensional functions such as images. For functional PLS, we use the penalized PLS (PPLS hereafter) by \cite{ppls} for its best performance. The codes for KDA and PDA are downloaded from the author's website of \cite{LSSS}; the code for the B-spline based LDA is download from the author's website; the authors of PPLS \cite{ppls} provide their code as an R \cite{R} package \emph{ppls}. The code for PPCA is implemented by our own.

In the first simulation, we consider a simple two-classes curve discrimination which most of the available methods work well. This example is adapted from \cite{PredaCS07}: assume that
\begin{align*}
X_1(t)&=Uh_1(t)+(1-U)h_2(t)+\epsilon(t) \\
X_2(t)&=Uh_1(t)+(1-U)h_3(t)+\epsilon(t)
\end{align*}
where $X_1(t)$ and $X_2(t)$ are random curves of categories $1$ and $2$, respectively, $U$ is a random variable from Uniform$(0,1)$, $h_1(t)=\max\{6-|t-11|,0\}$, $h_2(t)=h_1(t-4)$, $h_3(t)=h_1(t+4)$ with $t$ being 101 equidistant points from $[1,21]$, and $\epsilon(t)$'s are i.i.d. standard normal distributed random noise. We consider 100 simulated samples of size $N$, with $N/2$ observations in each class. For each sample we generate additional 200 observations for testing, with $100$ observations in each class. Figure \ref{Preda07} displays a sample of simulated curves for each class.  Table \ref{simulate1_error_tab}-\ref{simulate1_SVM} present the averaged error rates with standard deviations by different approaches. We can observe that KDA and kernel SVM have significantly worse error rates regardless of the training sample size $N$, even when $N$ is large. This suggests that kernel tricks are not as efficient as the other functional approaches when data are observed form functions. Among different functional approaches, B-spline based LDA does not perform as good as the other competitors. This is not surprising since as the authors mentioned in their article, their approach may sometimes be confounding and could be trapped to local minima. All the other functional approaches share similar performances in this simulation, while our GP-LDA slightly outperforms the other approaches regardless of the sample size $N$.

\begin{figure}[!t]
\centering
\includegraphics[width=0.8\linewidth]{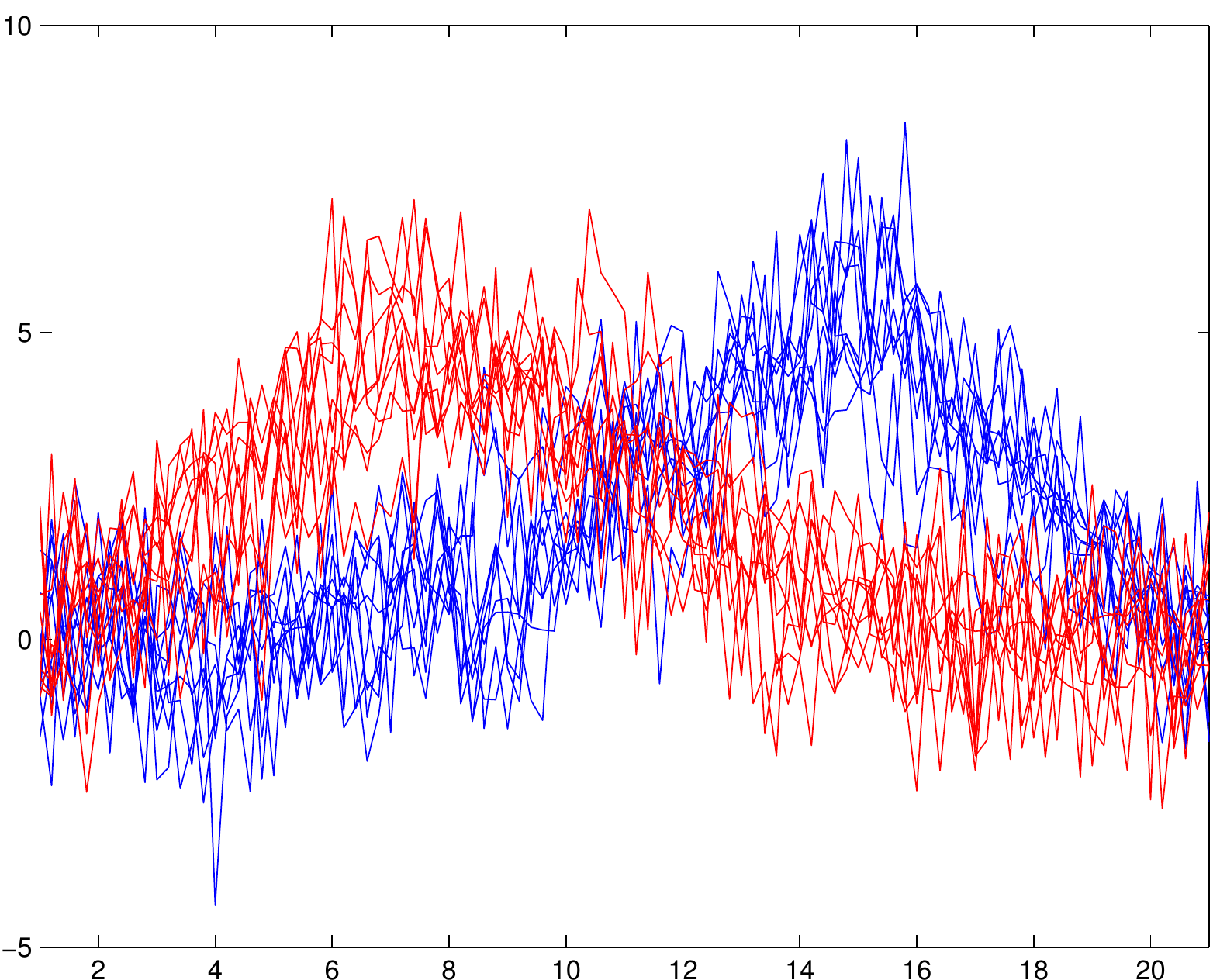}
\caption{Sample curves of the first simulation}
\label{Preda07}
\end{figure}

\begin{table}[]
\centering
\begin{tabular}{|r|ccccc|}
\hline
N & PPCA+LDA & PPLS+LDA & B-spline & PDA & GP-LDA \\
\hline \hline
50 & 4.53$\pm$2.01 & 4.71$\pm$3.06 & 5.72$\pm$1.92 & 4.73$\pm$2.10 & \bf{4.20$\pm$2.01} \\
200 & 2.75$\pm$1.53 & 2.87$\pm$1.56 & 3.63$\pm$1.34 & 3.05$\pm$1.78 & \bf{2.68$\pm$1.51} \\
800 & 2.41$\pm$1.11 & 2.41$\pm$1.11 & 2.75$\pm$1.3 & 2.41$\pm$1.27 & \bf{2.39$\pm$1.13} \\
\hline
\end{tabular}
\caption{Misclassification rates (mean$\pm$std\%) for different functional approaches in the first simulation.}
\label{simulate1_error_tab}
\end{table}

\begin{table}[]
\centering
\begin{tabular}{|r|cccc|}
\hline
N & RBF & poly (3) & poly (4) & poly (5)  \\
\hline \hline
50 & 4.89$\pm$1.82 & 4.55$\pm$1.77 & 4.81$\pm$1.86 & 4.56$\pm$1.85\\
200 & 4.02$\pm$1.66 & 3.80$\pm$1.61 & 4.20$\pm$1.52 & 3.98$\pm$1.58 \\
800 & 2.70$\pm$1.51 & 2.83$\pm$1.55 & 2.80$\pm$1.47 & 2.85$\pm$1.52 \\
\hline
\end{tabular}
\caption{Misclassification rates (mean$\pm$std\%) for KDA with different kernels in the first simulation. The radius basis kernel is denoted by RBF, and polynomial kernels of order $d$ are denoted by poly ($d$).}
\label{simulate1_kda}
\end{table}

\begin{table}[]
\centering
\begin{tabular}{|r|ccccc|}
\hline
N & linear & RBF & poly (3) & poly (4) & poly (5)  \\
\hline \hline
50 & 4.71$\pm$2.17 & 5.03$\pm$1.84 & 5.84$\pm$3.31 & 24.47$\pm$4.30 & 19.69$\pm$6.87 \\
200 & 3.22$\pm$1.33 & 4.30$\pm$1.66 & 4.11$\pm$1.32 & 18.38$\pm$3.84 & 8.91$\pm$4.75 \\
800 & 2.69$\pm$1.20 & 2.91$\pm$1.41 & 2.90$\pm$1.22 & 4.05$\pm$2.83 & 2.91$\pm$1.38 \\
\hline
\end{tabular}
\caption{Misclassification rates (mean$\pm$std\%) for SVM with different kernels in the first simulation. The radius basis kernel is denoted by RBF, and polynomial kernels of order $d$ are denoted by poly ($d$).}
\label{simulate1_SVM}
\end{table}

In the second simulation, we consider the case that (functional) PCA will fail. Assume that
\begin{align*}
X_1(t)&=\sin(2\pi t)/4+Z\cdot\sin(4\pi t)+\epsilon(t) \\
X_2(t)&=Z\cdot\sin(4\pi t)+\epsilon(t)
\end{align*}
where $X_1(t)$ and $X_2(t)$ are random curves of class 1 and 2, $Z$ is a standard normal r.v., $t$'s are 100 equidistant points from $[0,1]$, and $\epsilon(t)$'s are i.i.d. normally distributed random noise with mean $0$ and variance $0.1$. Notice that in this case the mean difference between class 1 and 2 is $\sin(2\pi t)/4$, which is orthogonal to the common basis function $\sin(4\pi t)$. Eigen-decomposition to the total covariance function yields to two eigenfuncitons: $\sin(4\pi t)$ with eigenvalue $1$ and $\sin(2\pi t)$ with eigenvalue $1/16$. Thus PCA tends to pick up $\sin(4\pi t)$ as the first principal component and may neglect $\sin(2\pi t)$ since it explains only less than $6\%$ of total variance.  However, projecting both $X_1(t)$ and $X_2(t)$ to the first principal component $\sin(4\pi t)$ gives the same result $Z\sim N(0,1)$ and hence cannot be discriminated. We consider 100 simulated samples of size $N$, with $N/2$ observations in each class. For each sample we generate additional 200 observations for testing, with $100$ observations in each class. Figure \ref{Yang13} displays a sample of simulated curves for each class. Table \ref{simulate2_error_tab}-\ref{simulate2_SVM} present the averaged error rates with standard deviations by different approaches. As we expected, LDA after functional PCA has similar performances to merely random guesses, while functional PLS can find basis functions that are somehow discriminative. KDA and kernel SVM also perform poorly in this case regardless which kernel is used, which confirms our argument again that kernel tricks are helpless for functional data. In this case B-spline based LDA, PDA and our GP-LDA have significantly lower misclassification rates to the other approaches regardless of sample size $N$, and our GP-LDA still has the best performance.

\begin{figure}[!t]
\centering
\includegraphics[width=0.8\linewidth]{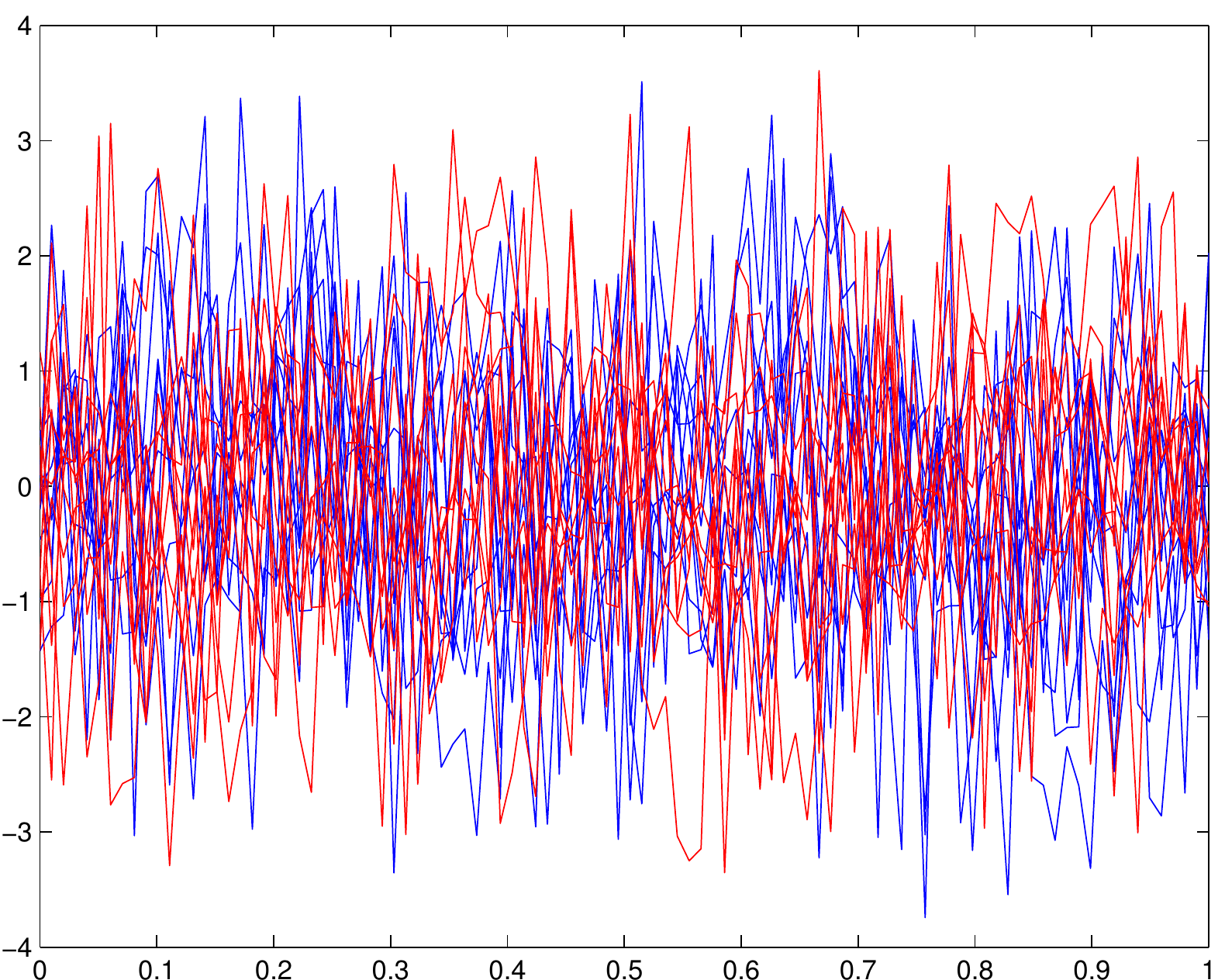}
\caption{Sample curves of the second simulation}
\label{Yang13}
\end{figure}

\begin{table}[]
\centering
\begin{tabular}{|r|ccccc|}
\hline
N & PPCA & PPLS &B-spline & PDA & GP-LDA \\
\hline \hline
20 & 48.8$\pm$4.50 & 46.9$\pm$4.84 & 44.33$\pm$4.96 & 43.78$\pm$5.89 & \bf{41.75$\pm$4.12} \\
50 & 50.00$\pm$3.80 & 45.41$\pm$5.21 & 41.76$\pm$4.18 & 40.73$\pm$4.15 & \bf{39.6$\pm$3.46} \\
200 & 50.00$\pm$3.12 & 39.25$\pm$4.76 & 38.36$\pm$3.22 & 37.68$\pm$3.36 & \bf{36.83$\pm$3.21} \\
\hline
\end{tabular}
\caption{Misclassification rates (mean$\pm$std\%) for the second simulation.}
\label{simulate2_error_tab}
\end{table}

\begin{table}[]
\centering
\begin{tabular}{|r|cccc|}
\hline
N & RBF & poly (3) & poly (4) & poly (5)  \\
\hline \hline
20 & 48.55$\pm$3.54 & 48.86$\pm$3.50 & 50.05$\pm$3.24 & 49.03$\pm$3.56 \\ 
50 & 47.05$\pm$4.52 & 47.45$\pm$4.13 & 49.70$\pm$3.29 & 48.95$\pm$3.80 \\
200 & 45.28$\pm$3.73 & 45.05$\pm$3.81 & 49.77$\pm$3.55 & 46.06$\pm$3.82 \\
\hline
\end{tabular}
\caption{Misclassification rates (mean$\pm$std\%) for KDA with different kernels in the second simulation. The radius basis kernel is denoted by RBF, and polynomial kernels of order $d$ are denoted by poly ($d$).}
\label{simulate2_kda}
\end{table}

\begin{table}[]
\centering
\begin{tabular}{|r|ccccc|}
\hline
N & linear & RBF & poly (3) & poly (4) & poly (5)  \\
\hline \hline
20 & 48.57$\pm$3.73 & 48.73$\pm$3.43 & 49.51$\pm$2.71 & 50.35$\pm$2.18 & 49.98$\pm$1.98 \\
50 & 47.43$\pm$3.86 & 47.43$\pm$4.18 & 48.72$\pm$3.07 & 50.00$\pm$2.29 & 49.62$\pm$1.60 \\
200 & 45.28$\pm$3.60 & 45.28$\pm$3.91 & 46.08$\pm$3.74 & 50.01$\pm$2.39 & 49.00$\pm$2.20 \\
\hline
\end{tabular}
\caption{Misclassification rates (mean$\pm$std\%) for SVM with different kernels in the second simulation. The radius basis kernel is denoted by RBF, and polynomial kernels of order $d$ are denoted by poly ($d$).}
\label{simulate2_SVM}
\end{table}

\section{Real Applications}\label{realapp}

In this section we demonstrate the performances of the the proposed GP-LDA on four different tasks, namely: speech recognition, spectrum classification, face recognition and object categorization. The Phoneme \cite{pda,est2} dataset is used for the speech recognition task, which consists 4509 speech frames of five phonemes (872 frames for "she", 757 frames for "dark", 1163 frames for the vowel in "she", 695 frames for the vowel in "dark", and 1022 frames for the first vowel in "water".) All the speech frames were transformed into log-periodograms of length 256, which is one of several widely used methods for casting speech data in a form suitable for speech recognition. The Phoneme dataset can be found in the website of \cite{est2}. Fig. \ref{phoneme-samples} illustrates some sample log-periodograms of the Phoneme dataset.

\begin{figure}[!t]
\centering
\includegraphics[width=1\linewidth]{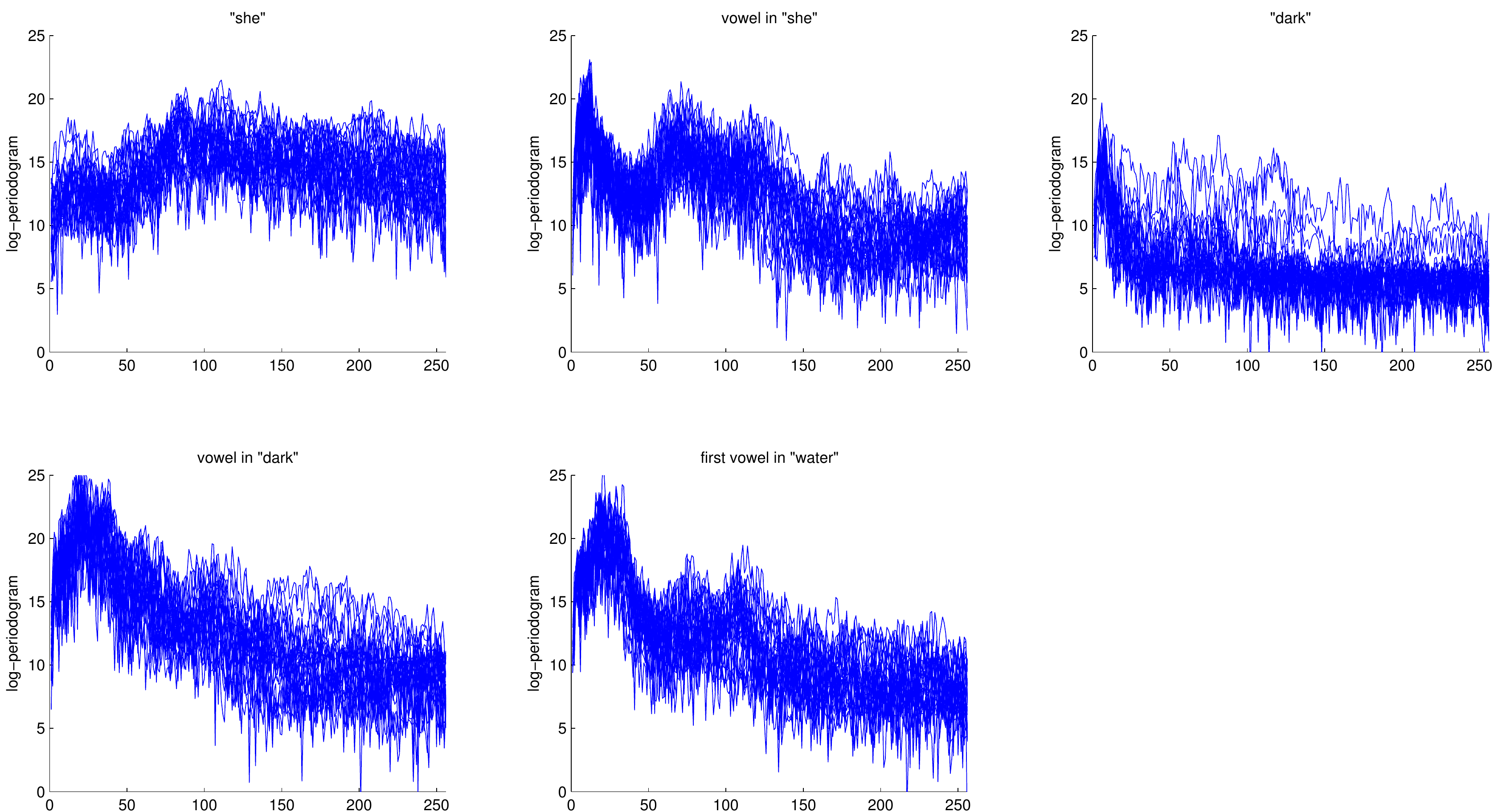}
\caption{Sample log-periodograms of the Phoneme dataset.}
\label{phoneme-samples}
\end{figure}

The wheat moisture spectra dataset \cite{Kalivas97} is used for the spectrum classification task. This dataset contains near-infrared reflectance spectra of 100 wheat samples, measured in 2 nm intervals from 1100 to 2500nm (of length 701), as well as the samples' moisture content. We treat the samples whose moisture content are less than 14 as low-moisture samples (41/100), while those whose moisture content are larger than 14 as high-moisture samples (59/100). The wheat moisture spectra dataset can be found in the R package \emph{fds} \cite{fds}. Fig. \ref{wheat-spectra} illustrates the spectra of the wheat moisture spectra dataset.

\begin{figure}[!t]
\centering
\includegraphics[width=0.6\linewidth]{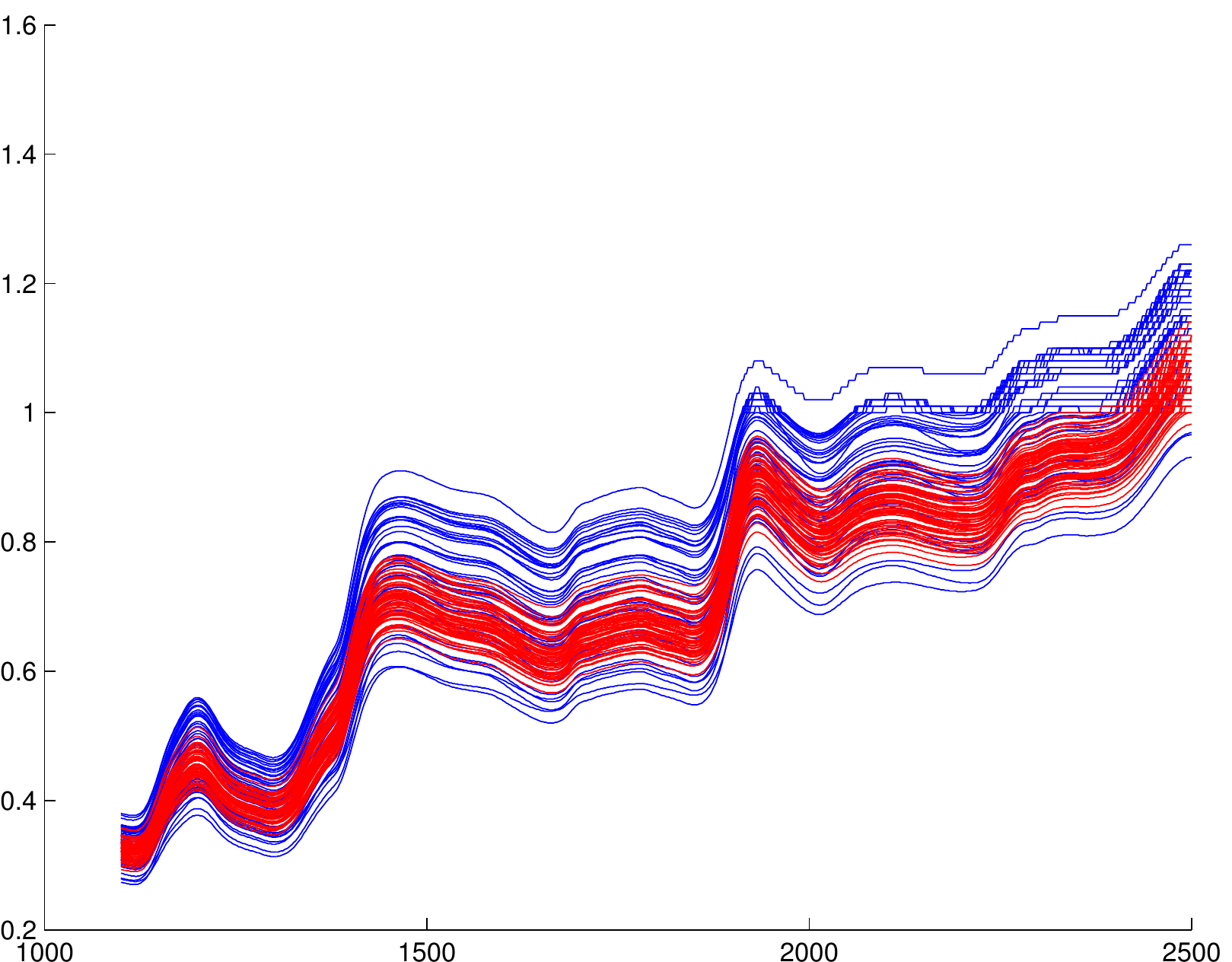}
\caption{The spectra of the wheat moisture spectra dataset. Blue: low-moisture samples; red: high-moisture samples.}
\label{wheat-spectra}
\end{figure}

The Yale database \cite{croppedYale} is used for the face recognition task, which contains 165 gray scale images of 15 individuals, each of 11 images with different lighting conditions and facial expressions (normal, happy, sad, sleepy, surprised, and wink). The ETH-80 \cite{ETH80} dataset is used for the object categorization, which contains 3280 images of 8 categories. Each category contains 10 different objects with 41 views per object. All the images in both datasets are well aligned and cropped. Each cropped image is resized to $32\times32$ pixels, with 256 gray levels per pixel. We rescale the pixel values to $[0,1]$ (divided by 255). Sample images of Yale and ETH-80 database are shown in Fig. \ref{Yale-images} and \ref{ETH80-images}.

\begin{figure}[!t]
\centering
\includegraphics[width=1\linewidth]{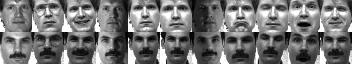}
\caption{Some sample images from Yale database}
\label{Yale-images}
\end{figure}

\begin{figure}[!t]
\centering
\includegraphics[width=0.6\linewidth]{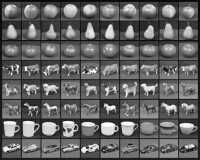}
\caption{Some sample images from ETH-80 data set}
\label{ETH80-images}
\end{figure}

Each dataset is then partitioned into the gallery (training) and probe (testing) set with different numbers. For the Phoneme dataset, $m$ log-periodograms per phoneme are randomly selected for training, and all of the remaining log-periodograms are used for testing. For the wheat moisture spectra dataset, $m$ spectra per class are randomly selected for training, and all of the remaining spectra are used for testing. For the Yale database, $m$ images per person are randomly selected for training, and the remaining images are for testing. For the ETH-80 dataset, $m$ images per \emph{category} are randomly selected for training, and the remaining images are for testing. Note that in our setting the training set may not contain all the 10 objects for each category.

For speech recognition and spectrum classification tasks, we compare our GP-LDA method (Section \ref{gplda}) with the following algorithms: KDA with different kernel functions (radius basis and polynomial kernels of order 3, 4, and 5), PDA, PPCA+LDA, and PPLS+LDA. For image recognition tasks, we compare our GP-LDA method (Section \ref{gplda}) with the following algorithms: Fisherface \cite{fisherface} as the baseline, KDA with different kernel functions, PPCA+LDA, and PDA. PPLS is excluded here since the code available does not support dealing with images. For all the tasks, SVM different kernel functions (linear, radius basis, and polynomial kernels of order 3, 4, and 5) are also used for comparison. The tuning parameters for our GP-LDA are estimated as described in Section \ref{gplda} with hyper-parameters $a_1=a_2=1$, $b_1=20$, and $b_2=100$ to achieve better performance. The tuning parameters for the other methods (if any) are selected by cross-validations.

The code for Fisherface is downloaded from the author's website of \cite{LSSS}. For SVM, we use the R package \emph{e1071} \cite{e1071}, a R-wrapper for LIBSVM \cite{cc01a}.

\subsection{Speech Recognition}

The misclassification rates of different algorithms on Phoneme dataset are listed in Tables \ref{Phoneme_results}-\ref{Phoneme_svm}. Table \ref{Phoneme_results} lists the misclassification rates of different functional approaches. From this table we can observe that PPCA+LDA does not work well on this dataset, while PDA and PPLS+LDA share similar results. The proposed GP-LDA method significantly outperforms all the other functional approaches, especially when the sample size becomes larger.

Tables \ref{Phoneme_kda} and \ref{Phoneme_svm} list the misclassification rates of KDA and SVM with different kernels. We can observe that both KDA and SVM with RBF kernel provide reasonable performances on this dataset, while the polynomial kernels with various orders perform poorly, especially as the order increases. Although RBF kernel provides slightly better results to linear kernel, however, the difference is not significant. Finally, our GP-LDA still significantly outperforms both KDA and SVM on this dataset, no matter which kernel function is used.

\begin{table}[htdp]
\begin{center}
\begin{tabular}{|r|cccc|}
\hline
Gallery Size & PDA & PPCA+LDA & PPLS+LDA & GP-LDA \\
\hline \hline
25 & 10.94$\pm$0.98 & 17.29$\pm$1.49 & 10.93$\pm$1.66 & \textbf{10.30$\pm$1.06} \\
50 & 10.20$\pm$0.81 & 16.38$\pm$0.78 & 10.02$\pm$0.53 & \textbf{8.54$\pm$0.51} \\
100 & 10.22$\pm$0.73 & 16.33$\pm$0.62 & 9.69$\pm$0.48 & \textbf{7.98$\pm$0.46} \\
\hline
\end{tabular}
\end{center}
\caption{Misclassification rates for PDA, PPCA+LDA, PPLS+LDA, and the proposed GP-LDA (mean$\pm$std\%) on the Phoneme dataset.}
\label{Phoneme_results}
\end{table}

\begin{table}[htdp]
\begin{center}
\begin{tabular}{|r|cccc|}
\hline
Gallery Size & RBF & poly (3) & poly (4) & poly (5) \\
\hline \hline
25 & 10.70$\pm$1.05 & 22.67$\pm$2.32 & 28.25$\pm$2.53 & 34.72$\pm$2.61 \\
50 & 10.47$\pm$0.81 & 21.96$\pm$1.96 & 26.44$\pm$1.77 & 32.43$\pm$2.08 \\
100 & 10.32$\pm$0.65 & 17.94$\pm$1.12 & 22.51$\pm$1.30 & 28.55$\pm$1.38 \\
\hline
\end{tabular}
\end{center}
\caption{Misclassification rates (mean$\pm$std\%) for KDA with different kernels on the Phoneme dataset. The radius basis kernel is denoted by RBF, and polynomial kernels of order $d$ are denoted by poly ($d$).}
\label{Phoneme_kda}
\end{table}

\begin{table}[htdp]
\begin{center}
\begin{tabular}{|r|ccccc|}
\hline
Gallery Size & linear & RBF & poly (3) & poly (4) & poly (5) \\
\hline \hline
25 & 11.22$\pm$1.21 &  10.69$\pm$1.22 & 13.32$\pm$2.38 & 22.97$\pm$4.52 & 24.15$\pm$5.96 \\
50 & 10.49$\pm$0.90 &  9.88$\pm$0.84 & 10.51$\pm$0.90 & 17.00$\pm$2.51 & 17.06$\pm$4.16 \\
100 & 10.48$\pm$0.71 & 9.85$\pm$0.63 & 10.50$\pm$0.59 & 13.58$\pm$1.36 & 13.02$\pm$2.54 \\
\hline
\end{tabular}
\end{center}
\caption{Misclassification rates (mean$\pm$std\%) for SVM with different kernels on the Phoneme dataset. The radius basis kernel is denoted by RBF, and polynomial kernels of order $d$ are denoted by poly ($d$).}
\label{Phoneme_svm}
\end{table}

\subsection{Spectra Classification}

The misclassification rates of different algorithms on the wheat moisture spectra dataset are listed in Tables \ref{wheat_results}-\ref{wheat_svm}. PPCA+LDA does not perform well on this dataset, while the other functional approaches provide reasonable results. Linear SVM also provides satisfying performance on this dataset. However, kernel tricks may not be appropriate for this dataset: neither KDA nor kernel SVM provide reasonable results on this dataset, no matter which kernel function is used.

\begin{table}[htdp]
\begin{center}
\begin{tabular}{|r|ccccc|}
\hline
Gallery Size & PDA & PPCA+LDA & PPLS+LDA & GP-LDA & \\
\hline \hline
20 & 1.77$\pm$1.95 & 30.25$\pm$5.75 & 0.25$\pm$0.18 & \textbf{0.13$\pm$0.15} & \\
\hline
\end{tabular}
\end{center}
\caption{Misclassification rates (mean$\pm$std\%) on the wheat moisture spectra dataset.}
\label{wheat_results}
\end{table}

\begin{table}[htdp]
\begin{center}
\begin{tabular}{|r|cccc|}
\hline
Gallery Size & RBF & poly (3) & poly (4) & poly (5) \\
\hline \hline
20 & 2.66$\pm$0.71 & 4.40$\pm$1.93 & 7.60$\pm$4.12 & 11.00$\pm$4.72 \\
\hline
\end{tabular}
\end{center}
\caption{Misclassification rates (mean$\pm$std\%) for KDA with different kernels on the wheat moisture spectra dataset. The radius basis kernel is denoted by RBF, and polynomial kernels of order $d$ are denoted by poly ($d$).}
\label{wheat_kda}
\end{table}

\begin{table}[htdp]
\begin{center}
\begin{tabular}{|r|ccccc|}
\hline
Gallery Size & linear & RBF & poly (3) & poly (4) & poly (5) \\
\hline \hline
20 & 0.37$\pm$0.15 & 1.70$\pm$0.32 & 5.78$\pm$1.35 & 7.65$\pm$3.91 & 11.50$\pm$4.52 \\
\hline
\end{tabular}
\end{center}
\caption{Misclassification rates (mean$\pm$std\%) for SVM with different kernels on the wheat moisture spectra dataset. The radius basis kernel is denoted by RBF, and polynomial kernels of order $d$ are denoted by poly ($d$).}
\label{wheat_svm}
\end{table}

\subsection{Face Recognition}

The misclassification rates of different algorithms on Yale database are listed in Tables \ref{Yale_Constant}-\ref{Yale_SVM}. Note that the performances for PDA in table \ref{Yale_Constant} are better than the result reported in the original paper \cite{LSSS}. Table \ref{Yale_Constant} shows that PPCA+LDA works slightly better than Fisherface when the training sample size is small; however, when the training sample size becomes larger, it shares similar performances with Fisherface. PDA significantly outperforms Fisherface and PPCA+LDA, especially when the training sample size becomes larger. This implies that functional assumption does provide significant improvement on this dataset. The proposed GP-LDA method generally outperforms all the other approaches when the training sample size is moderately large, which suggests that proper modeling of functional assumption is really important for this dataset.

The misclassification rates of KDA and SVM with different kernels are listed in tables \ref{Yale_KDA} and \ref{Yale_SVM}. We can observe that KDA with RBF kernel does not outperform Fisherface, while KDA with polynomial kernels provide worse results as the order increases. Oh the other hand, for SVM the linear kernel, RBF kernel, and polynomial kernel with order 3 provide similar results, and polynomial kernels with orders 4 and 5 perform significantly better. Finally, the proposed GP-LDA in table \ref{Yale_Constant} still provide significantly better results to both KDA and SVM no matter which kernel function is used, especially when the training sample size is moderately large.

\begin{table}[htdp]
\begin{center}
\begin{tabular}{|r|cccc|}
\hline
Gallery Size & Fisherface & PPCA & PDA & GP-LDA \\
\hline \hline
2 & 44.44$\pm$4.6 & 41.12$\pm$4.62 & 40.49$\pm$4.99 & 34.89$\pm$4.8 \\
3 & 33.82$\pm$4.17 & 32.49$\pm$4.57 & 28.22$\pm$3.87 & 23.37$\pm$3.85 \\
4 & 27.75$\pm$4.79 & 27.40$\pm$3.78 & 22.82$\pm$4.03 & \textbf{18.23$\pm$3.53} \\
5 & 23.80$\pm$4.38 & 23.44$\pm$3.85 & 17.62$\pm$3.46 & \textbf{14.24$\pm$3.79} \\
6 & 20.61$\pm$4.13 & 21.00$\pm$4.54 & 14.03$\pm$4.10 & \textbf{11.55$\pm$3.58} \\
7 & 19.73$\pm$4.35 & 19.73$\pm$4.35 & 12.90$\pm$4.04 & \textbf{10.4.$\pm$3.68} \\
8 & 18.31$\pm$4.22 & 18.31$\pm$4.22 & 10.58$\pm$4.80 & \textbf{8.04$\pm$4.49} \\
\hline
\end{tabular}
\end{center}
\caption{Misclassification rates for PDA, PPCA+LDA, PPLS+LDA, and the proposed GP-LDA (mean$\pm$std\%) on Yale database.}
\label{Yale_Constant}
\end{table}

\begin{table}[htdp]
\begin{center}
\begin{tabular}{|r|cccc|}
\hline
Gallery Size & RBF & poly (3) & poly (4) & poly (5) \\
\hline \hline
2 & 44.77$\pm$3.85 & 51.54$\pm$4.89 & 56.76$\pm$5.58 & 61.51$\pm$5.77\\
3 & 34.70$\pm$3.65 & 42.53$\pm$4.60 & 48.85$\pm$4.48 & 54.43$\pm$4.96\\
4 & 29.10$\pm$4.09 & 36.04$\pm$3.99 & 42.99$\pm$4.27 &49.37$\pm$4.51 \\
5 & 23.67$\pm$3.69 & 31.71$\pm$4.09 & 38.27$\pm$4.17 & 45.38$\pm$5.10 \\
6 & 22.24$\pm$3.91 & 29.04$\pm$3.86 & 35.71$\pm$4.16 & 42.53$\pm$4.52 \\
7 & 20.19$\pm$3.95 & 23.33$\pm$4.34 & 33.70$\pm$4.95 & 40.43$\pm$4.59 \\
8 & 18.36$\pm$4.28 & 25.51$\pm$5.14 & 31.42$\pm$5.42 & 38.80$\pm$5.58 \\
\hline
\end{tabular}
\end{center}
\caption{Misclassification rates (mean$\pm$std\%) for KDA with different kernels on Yale database. The radius basis kernel is denoted by RBF, and polynomial kernels of order $d$ are denoted by poly ($d$).}
\label{Yale_KDA}
\end{table}

\begin{table}[htdp]
\begin{center}
\begin{tabular}{|r|ccccc|}
\hline
Gallery Size & linear & RBF & poly (3) & poly (4) & poly (5) \\
\hline \hline
2 & 44.33$\pm$4.27 & 44.76$\pm$4.20 & 48.18$\pm$5.77 & 35.31$\pm$5.30 & \textbf{34.09 $\pm$5.58} \\
3 & 37.68$\pm$4.09 & 37.63$\pm$3.92 & 44.10$\pm$5.75 & 28.94$\pm$3.87 & \textbf{22.94$\pm$4.47} \\
4 & 32.50$\pm$2.95 & 32.36$\pm$3.18 & 40.83$\pm$4.89 & 25.92$\pm$5.25 & 20.16$\pm$5.64 \\
5 & 29.07$\pm$3.24 & 28.64$\pm$3.71 & 35.14$\pm$5.74 & 21.63$\pm$4.60 & 16.18$\pm$4.58 \\
6 & 27.87$\pm$3.83 & 27.57$\pm$3.89 & 30.88$\pm$4.44 & 19.34$\pm$3.68 & 14.90$\pm$2.69 \\
7 & 25.87$\pm$4.18 & 26.571$\pm$4.40 & 26.29$\pm$5.00 & 15.62$\pm$3.58 & 13.33$\pm$4.90 \\
8 & 23.33$\pm$5.12 & 23.64$\pm$5.27 & 20.60$\pm$5.25 & 12.17$\pm$3.53 & 16.46$\pm$4.90 \\
\hline
\end{tabular}
\end{center}
\caption{Misclassification rates (mean$\pm$std\%) for SVM with different kernels on Yale database. The radius basis kernel is denoted by RBF, and polynomial kernels of order $d$ are denoted by poly ($d$).}
\label{Yale_SVM}
\end{table}

\subsection{Object Categorization}

The misclassification rates of different algorithms on ETH-80 dataset are listed in Tables \ref{ETH80_results}-\ref{ETH80_SVM}. From table \ref{ETH80_results} we can find that PDA significantly outperforms Fisherface and KDA, and PPCA+LDA works almost as good as PDA on this dataset. The proposed GP-LDA method significantly outperforms all the other functional approaches. 

From table \ref{ETH80_KDA} we can see that the selection of kernels is not important for KDA on this dataset. Furthermore, KDA works better than Fisherface, but it is still not as good as functional approaches no matter which kernel is used.

Table \ref{ETH80_SVM} shows interesting results. SVM with linear and RBF kernels share similar results, and both of them work much better than all the LDA-based methods including KDA and our GP-LDA. This suggests that LDA-based methods may not be appropriate for this dataset. This could happen, for example, when the distribution of this dataset is skewed, etc.

\begin{table}[htdp]
\begin{center}
\begin{tabular}{|r|cccc|}
\hline
Gallery Size & Fisherface  & PPCA & PDA & GP-LDA \\
\hline \hline
20 & 48.08$\pm$1.9 & 38.56$\pm$1.68 & 34.70$\pm$1.84 & 30.32$\pm$1.62 \\
100 & 41.25$\pm$1.58 & 28.26$\pm$1.79 & 28.56$\pm$1.44 & 19.61$\pm$1.08 \\
\hline
\end{tabular}
\end{center}
\caption{Misclassification rates for PDA, PPCA+LDA, PPLS+LDA, and the proposed GP-LDA (mean$\pm$std\%) on ETH-80 dataset.}
\label{ETH80_results}
\end{table}

\begin{table}[htdp]
\begin{center}
\begin{tabular}{|r|cccc|}
\hline
Gallery Size & RBF & poly (3) & poly (4) & poly (5) \\
\hline \hline
20 & 45.29$\pm$3.81 & 45.55$\pm$3.63 & 45.21$\pm$3.63 & 45.57$\pm$3.65 \\
100 & 31.86$\pm$3.45 & 29.94$\pm$2.80 & 29.90$\pm$2.81& 30.43$\pm$2.86 \\
\hline
\end{tabular}
\end{center}
\caption{Misclassification rates (mean$\pm$std\%) for KDA with different kernels on ETH-80 dataset. The radius basis kernel is denoted by RBF, and polynomial kernels of order $d$ are denoted by poly ($d$).}
\label{ETH80_KDA}
\end{table}

\begin{table}[htdp]
\begin{center}
\begin{tabular}{|r|ccccc|}
\hline
Gallery Size & linear & RBF & poly (3) & poly (4) & poly (5) \\
\hline \hline
20 & 26.67$\pm$1.42  & \textbf{24.33$\pm$1.75} & 37.31$\pm$2.93 & 48.69$\pm$3.06 & 50.94$\pm$2.94 \\
100 & 15.24$\pm$0.71 & \textbf{14.37$\pm$0.74} & 16.67$\pm$1.01 & 26.45$\pm$1.43  & 32.20$\pm$1.62 \\
\hline
\end{tabular}
\end{center}
\caption{Misclassification rates (mean$\pm$std\%) for SVM with different kernels on ETH-80 dataset. The radius basis kernel is denoted by RBF, and polynomial kernels of order $d$ are denoted by poly ($d$).}
\label{ETH80_SVM}
\end{table}

\section{Conclusion Remarks}\label{conclusion}

In this paper, we introduce a Bayesian framework of Gaussian process for extending Fisher's discriminant to the data observed form smooth functions. Inspired from Bayesian smoothing, the smoothness assumption of observed functions are translated to priors on mean and the common covariance functions. Then, we derive our novel GP-LDA method as a maximum a posteriori probability (MAP) estimate within this framework. The advantage of our approach are of two-fold: first we introduce a more theoretical framework to incorporate the information of spatial-correlations for functional data. Second, within the Bayesian framework the tuning parameters can be estimated simultaneously with the other unknown parameters.  Experimental results on simulated and real datasets show that the proposed GP-LDA does outperform previous methods. Finally, our Bayesian Gaussian Process framework can easily combine with other extensions of Fisher's discriminant as well, such as multilinear LDA, MFA, etc.

\bibliographystyle{plain}
\bibliography{gplda-arxiv}

\end{document}